\documentclass[english,envcountsame]{llncs}
\usepackage[T1]{fontenc}
\usepackage[latin9]{inputenc}
\usepackage{array}
\usepackage{verbatim}
\usepackage{float}
\usepackage{amsmath}
\usepackage{amssymb}
\usepackage{stmaryrd}
\usepackage{setspace}
\makeatletter

\providecommand{\tabularnewline}{\\}

\usepackage{xcolor}
\definecolor{darkred}{rgb}{0.8,0,0} 
\usepackage{hyperref}
\hypersetup{final, colorlinks=true, linkcolor=black, urlcolor=black, citecolor=black, linkbordercolor=black, pdfborderstyle={/S/U/W 1}}
\usepackage{setspace}
\usepackage[absolute,overlay]{textpos} 
\usepackage{listings}
\usepackage{cite}
\usepackage{tikz}
\usetikzlibrary{automata,arrows,shapes,decorations,topaths,trees,backgrounds,shadows,positioning,fit,calc}

\pgfdeclarelayer{background}
\pgfdeclarelayer{foreground}
\pgfsetlayers{background,main,foreground}

\makeatletter
\g@addto@macro \normalsize {%
 \setlength\abovedisplayskip{4pt plus 2pt minus 2pt}%
 \setlength\belowdisplayskip{4pt plus 2pt minus 2pt}%
}
\makeatother

\DeclareRobustCommand{\VAN}[2]{#2}

\usepackage{nameref}
\makeatletter
\let\orgdescriptionlabel\descriptionlabel
\renewcommand*{\descriptionlabel}[1]{%
  \let\orglabel\label
  \let\label\@gobble
  \phantomsection
  \edef\@currentlabel{#1}%
  \let\label\orglabel
  \orgdescriptionlabel{#1}%
}
\makeatother

\makeatletter
\newcommand*{\textlabel}[2]{%
  \edef\@currentlabel{#1}
  \phantomsection
  #1\label{#2}
}
\makeatother

\AtBeginDocument{
  
}

\makeatother

\usepackage{babel}
\begin{document}
\setstretch{1}
\title{Awareness Logic: Kripke Lattices as a Middle Ground between Syntactic and Semantic Models}
\author{Gaia Belardinelli and Rasmus K. Rendsvig \institute{Center for Information and Bubble Studies, University of Copenhagen \\ \email{\{belardinelli,rasmus\}@hum.ku.dk}}}
\maketitle
\begin{abstract}
The literature on awareness modeling includes both syntax-free and
syntax-based frameworks. Heifetz, Meier \& Schipper (HMS) propose
a lattice model of awareness that is syntax-free. While their lattice
approach is elegant and intuitive, it precludes the simple option
of relying on formal language to induce lattices, and does not explicitly
distinguish uncertainty from unawareness. Contra this, the most prominent
syntax-based solution, the Fagin-Halpern (FH) model, accounts for
this distinction and offers a simple representation of awareness,
but lacks the intuitiveness of the lattice structure. Here, we combine
these two approaches by providing a lattice of Kripke models, induced
by atom subset inclusion, in which uncertainty and unawareness are
separate. We show our model equivalent to both HMS and FH models by
defining transformations between them which preserve satisfaction
of formulas of a language for explicit knowledge, and obtain completeness
through our and HMS' results. Lastly, we prove that the Kripke lattice
model can be shown equivalent to the FH model (when awareness is propositionally
determined) also with respect to the language of the Logic of General
Awareness, for which the FH model where originally proposed.
\end{abstract}

\section{Introduction}

Awareness has been intensively studied in logic and game theory since
its first formal treatment by Fagin and Halpern \cite{HalpernFagin88}.
In these fields, awareness is added as a complement to uncertainty
in models for knowledge and rational interaction. In short, where
uncertainty concerns an agent's ability to distinguish possible states
of the world based on its available information, awareness concerns
the agent's ability to even contemplate aspects of a state, where
such inability stems from the \emph{unawareness }of the concepts that
constitute said aspects. Thereby, models that include awareness avoid
problems of logical omniscience (at least partially) and allows modeling
game theoretic scenarios where the possibility of some action may
come as an utter surprise.

Several models of awareness have been proposed in the literature,
which either follow the \emph{semantic} (or \emph{syntax-free}) or
the \emph{syntactic} (or \emph{syntax-based}) tradition of awareness
modeling. In the semantic tradition, awareness is usually represented in Aumann-like
event structures, which are defined without appeal to atomic propositions
or other syntax. The awareness notion presented in these frameworks
inherits the syntax-free definition and is thus captured by a specific
subset of states.

An instance of this approach is given by Heifetz, Meier and Schipper
(HMS), who propose a lattice-based conceptualization of awareness
\cite{HMS2006}. The backbone of HMS' \emph{unawareness frames} is
a complete lattice of state-spaces $(\mathcal{S},\preceq)$, with
the intuition that the higher a space is, the richer the ``vocabulary''
it has to describe its states. Since the approach is syntax-free,
this intuition is not modeled using a formal language. It is represented
using $\preceq$ and a family of maps $r_{S}^{S'}$ which projects
state-space $S'$ down to $S$, with $r_{S}^{S'}(s)$ interpreted
as the representation of $s$ in the more limited vocabulary available
in $S$. Uncertainty and unawareness are captured \emph{jointly} by
a \emph{possibility correspondence $\Pi_{a}$} for each $a\in Ag$,
which maps a state weakly downwards to the set of states the agent
considers possible. If the mapped-to space is strictly less expressive,
this represents that the agent does not have full awareness of the
mapped-from state.

That HMS keep their model syntax-free is motivated in part by its
applicability in theoretical economics \cite[p. 79]{HMS2006}. We
think that their lattice-based conceptualization of awareness is both
elegant, interesting and intuitive, as it captures different levels
of awareness in a suggestive way. However, we also find its formalization
cumbersome. Exactly the choice to go fully syntax-free robs the model
of the option to rely on formal language to induce lattices and to
specify events, resulting in constructions which we find laborious
to deal with. This may, of course, be an artifact of us being accustomed
to non-syntax-free models used widely in epistemic logic.

Another artifact of our familiarity with epistemic logic models is
that we find HMS' joint definition of uncertainty and unawareness
difficult to relate to other formalizations of knowledge. When HMS
propose properties of their $\Pi_{a}$ maps, it is not clear to us
which aspects concern knowledge and which concern awareness. They
merge two dimensions which, to us, would be clearer if left separated.\footnote{As a reviewer of the short version of this paper \cite{BelardinelliRendsvig2020}
pointed out, then HMS take \emph{explicit }knowledge as foundational,
and derive awareness from it. This makes the one-dimensional representation
justified, if not even desirable. In contrast, epistemic logic models
are standardly interpreted as taking \emph{implicit }knowledge as
foundational. We think along the second line, and add awareness as
a second dimension. We are not taking a stand on whether one interpretation
is superior, but provide results to move between them.}

Moreover, while the HMS model allows agents to reason about their
unawareness, as possibility correspondences $\Pi_{a}$ provide them
with a subjective perspective, Halpern and Rêgo \cite{HalpernRego2008}
point out that the model includes no objective state, and so no outside
perspective. 

Alternatively, the literature has proposed syntactic approaches to
awareness modeling. The syntactic tradition has been initiated by
the seminal \cite{HalpernFagin88}, where Fagin and Halpern (FH) introduce
the Logic of General Awareness ($\Lambda_{LGA}$). Models for this
logic (FH models) are Kripke models $M=(W,R,V)$ augmented with an
\emph{awareness function $\mathcal{A}_{a}$}, for each agent $a\in Ag$,
that\emph{ }represents an agent $a$'s awareness at state $w$ by
assigning to $(a,w)$ a set of formulas\textemdash which is why these
models are called \emph{syntax-based}.

Since FH models represent uncertainty using the accessibility relation
$R$, as in standard epistemic logic, FH explicitly distinguish the
uncertainty and unawareness dimension. This allows for a versatile
representation of awareness, as, when the awareness function is not
otherwise restricted, an agent's awareness in a state can be any arbitrary
set of formulas. The FH approach has thus been inherited by a multitude
of models.

However, FH models lack the intuitiveness of the lattice structure,
and while Halpern and Rêgo argue that HMS models lack the objective
perspective, HMS \cite{HMS2006,Schipper2014} also argue that FH models
only present an outside perspective, as the full model must be taken
into account when assigning knowledge and awareness.\footnote{\cite{HalpernRego2008} argues that this boils down to a difference
in philosophical interpretation.}

In the present paper, we aim at combining the advantages of the HMS
and FH approaches. We propose to model awareness through a syntactically
induced lattice structure\textemdash primarily inspired by the HMS
model\textemdash where the awareness notion is captured through an
awareness map defined semantically. Roughly, we suggest to start from
a Kripke model $\mathtt{K}$ for a set of atoms $At$, spawn a lattice
containing restrictions of $\mathtt{K}$ to subsets of $At$, and
finally add maps $\pi_{a}$ on the lattice that take a world to a
copy of itself in a restricted model. This keeps the epistemic and
awareness dimensions separate: accessibility relations $R_{a}$ of
$\mathtt{K}$ encode uncertainty while maps $\pi_{a}$ encode awareness.

In this \emph{Kripke lattice model} both subjective and objective
perspectives are present: the starting Kripke model provides an outside
perspective on agents' knowledge and awareness, while the submodel
obtained by following $\pi_{a}$ presents agent $a$'s subjective
perspective. We remark further on this below.

Beyond the introduction of Kripke lattice models,\footnote{First introduced in the short version of this paper, \cite{BelardinelliRendsvig2020}.}
the main contribution of the paper is a set of technical results situating
these models with respect to the HMS and FH models. These comprise
three results about the equivalences of model classes (see Figure~\ref{fig:triangles}),
and as corollaries, two completeness results for Kripke lattice models.\vspace{-10pt}

\begin{figure}[H]
\begin{centering}
\resizebox{0.75\textwidth}{!}{
	\begin{tikzpicture}
\tikzset{world/.style={rectangle, text centered}
}

\node[world] (a) {KL};
\node[world, right= of a] (b) {FH};
\node[world, above= of a, xshift=24pt] (c) {HMS};
\node[world, below=of a, yshift=15pt, xshift=25pt] (!) {\scalebox{1}[1]{$\mathcal{L}$-equivalence}};

\draw (a) to  node[below] {\scalebox{0.5}[0.5]{{
			Prop.s 44,45}}} (b);
\draw (b) to node[right, xshift=-1pt]  {\scalebox{0.5}[0.5]{{\begin{tabular}{l}
			Halpern R\^ego 2008
			\end{tabular} }}} (c);
\draw (c) to   node[left] {\scalebox{0.5}[0.5]{{
			Prop.s 26,27}}} (a);	

\node[world, right=of b, xshift=100pt] (a) {KL};
\node[world, right= of a] (b) {FH};
\node[world, above= of a, xshift=24pt] (c) {HMS};
\node[world, below=of a, yshift=15pt, xshift=25pt] (!) {$\mathcal{L}^{KA}$-equivalence};

\draw (a) to  node[below] {\scalebox{0.5}[0.5]{{
			Prop.s 51,52}}}(b);
\path(b) to  node {$?$}(c);
\path (c) to node {$?$} (a);	
\end{tikzpicture}}
\par\end{centering}
\caption{\label{fig:triangles} \textbf{Known equivalence results between HMS,
FH and Kripke lattice (KL) models.} Left: \emph{$\mathcal{L}$}-equivalence
results between the model classes, two shown in this paper. Right:
$\mathcal{L}^{KA}$-equivalence between FH and KL models shown in
this paper, and the open issue of the correspondence between HMS and
the other two model classes with respect to $\mathcal{L}^{KA}$.}
\end{figure}
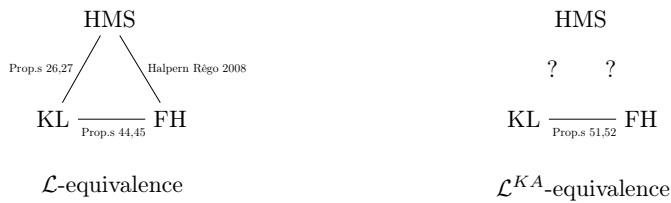

First, we show that, under three assumptions on $\pi_{a}$ and when
each $R_{a}$ is an equivalence relation, the Kripke lattice model
is $\mathcal{L}$-equivalent to the HMS model, in the sense that the
two satisfy the same formulas of the language of explicit knowledge
and awareness $\mathcal{L}$, defined below. Through this result and
the completeness of HMS logic $\Lambda_{HMS}$ with respect to the
class of HMS models, we obtain completeness of $\Lambda_{HMS}$ with
respect to the class $\boldsymbol{KLM}_{EQ}$ of Kripke lattice models
with equivalence relations.

Second, we show that $\boldsymbol{KLM}_{EQ}$ is $\mathcal{L}$-equivalent
to the class $\boldsymbol{S}$ of \emph{propositionally determined}
FH models with equivalence relations, as again the two satisfy the
same $\mathcal{L}$ formulas. %

Third, switching to use Kripke lattice models and FH models as semantics
for the language $\mathcal{L}^{KA}$ for implicit and explicit knowledge
and awareness\textemdash for which FH models were originally conceived\textemdash we
show that the class of Kripke lattices without restriction on the
accessibility relation and propositionally determined FH models are
$\mathcal{L}^{KA}$-equivalent. By FH's completeness result and our
model equivalence result, we show that the \emph{Logic of General
Awareness} $\Lambda_{LGA}$, which is based on $\mathcal{L}^{KA}$,
is also complete with respect to Kripke lattice models.

Jointly, these results firmly situates Kripke lattice models for awareness
with respect to the main existing models. Through detailed transformations
between the model classes, the results directly show correspondences
between the models' elements, and show that for both languages $\mathcal{L}$
and $\mathcal{L}^{KA}$, Kripke lattice models provide a rich semantic
framework, axiomatically characterizable by existing logics.

As Kripke lattice models are a novel construction, the paper's constructions
and results are new. However, the second result mentioned may also
be obtained through the first and an existing result by Halpern and
Rêgo \cite{HalpernRego2008}, that show that the class $\boldsymbol{S}$
of partitional, propositionally determined FH models is $\mathcal{L}$-equivalent
to the class of HMS models. We provide a direct proof of this result
as the involved transformation directly explicates the relationship
between Kripke lattices and FH models, used further to establish the
$\mathcal{L}^{KA}$-equivalence of these model classes.

The paper progresses as follows. Sections \ref{sec:HMS} and \ref{sec:Kripke-Lattice-Models}
present respectively the HMS model and our rendition, Kripke lattice
models. Section \ref{sec:Moving} introduces transformations between
the two models classes, and Section \ref{sec:Language} shows that
the transformations preserve formula satisfaction. Section \ref{sec:Logic}
presents a logic due to HMS \cite{HMS2008}, and shows, as a corollary
to our results, that it is complete with respect to our rendition.
Section \ref{sec:FH} introduces the FH model structure with respect
to language $\mathcal{L}$. As for the HMS model class, the next two
sections, Section \ref{sec:Moving btwn FH and Kripke lattices} and
\ref{sec:FH semantics}, presents the transformations between FH and
Kripke lattice models and show that they preserve formula satisfaction
with respect to language $\mathcal{L}$, respectively. Section \ref{sec:L^A}
presents the language $\mathcal{L}^{KA}$ on which the $\Lambda_{LGA}$
is based, and shows the equivalence of Kripke lattice and FH models
with respect to it, again using transformations. $\Lambda_{LGA}$
is introduces in Section \ref{sec:Logic-1}, where soundness and completeness
of $\Lambda_{LGA}$ over Kripke lattice models is shown. Section \ref{sec:Concluding-Remarks}
holds concluding remarks.\smallskip{}

Throughout the paper, we assume that $Ag$ is a finite, non-empty
set of agents, and that $At$ is a countable, non-empty set of atoms.

\section{\label{sec:HMS}The HMS Model}

This section presents HMS unawareness frames \cite{HMS2006}, their
syntax-free notions of knowledge and awareness, and their augmentation
with HMS valuations, producing HMS models \cite{HMS2008}. For context,
the HMS model is a multi-agent generalization of the Modica-Rustichini
model \cite{ModicaRustichini1999} which is equivalent to Halpern's
model in \cite{halpern2001alternative}, generalized by Halpern and
Rêgo to multiple agents \cite{HalpernRego2008}, resulting in a model
equivalent to the HMS model, cf. \cite{HMS2008}. See \cite{Schipper2014}
for an extensive~review.

The following definition introduces the basic structure underlying
the HMS model, as well as the properties of the $\Pi_{a}$ map that
controls the to-be-defined notions of knowledge and awareness. The
properties are described after Definition~\ref{def:unawareness-frame}.
Following Definition~\ref{def:HMS model} of HMS models, Figure~\ref{fig:HMSmodel}
illustrates a full HMS model, including its unawareness frame.
\begin{definition}
\label{def:unawareness-frame}An \textbf{unawareness frame} is a tuple
$\mathsf{F}=(\mathcal{S},\preceq,\mathcal{R},\Pi)$ where\smallskip{}

\noindent $(\mathcal{S},\preceq)$ is a complete lattice with $\mathcal{S}=\{S,S',...\}$
a set of disjoint, non-empty \textbf{state-spaces} $S=\{s,s',...\}$
s.t. $S\preceq S'$ implies $|S|\leq|S'|$. Let $\Omega_{\mathsf{F}}:=\bigcup_{S\in\mathcal{S}}S$
be the disjoint union of state-spaces~in~$\mathcal{S}$. For $X\subseteq\Omega_{\mathsf{F}}$,
let $S(X)$ be the state-space containing $X$, if such exists (else
$S(X)$ is undefined). Let $S(s)$ be $S(\{s\})$.\smallskip{}

\noindent $\mathcal{R}=\{r_{S}^{S'}\colon S,S'\in\mathcal{S},S\preceq S'\}$
is a family of \textbf{projections} $r_{S}^{S'}:S'\rightarrow S$.
Each $r_{S}^{S'}$ is surjective, $r_{S}^{S}$ is $Id$, and $S\preceq S'\preceq S''$
implies commutativity: $r_{S}^{S''}=r_{S}^{S'}\circ r_{S'}^{S''}$.
Denote $r_{S}^{T}(w)$ also by $w_{S}$.

$D^{\uparrow}=\bigcup_{S'\succeq S}(r_{S}^{S'})^{-1}(D)$ is the \textbf{upwards
closure }of $D\subseteq S\in\mathcal{S}$.\footnote{\textit{\emph{To avoid confusion, note that for $d\in S$, $(r_{S}^{S'})^{-1}(d)=\{s'\in S':r_{S}^{S'}(s')=d\}$
and for $D\subseteq S$, $(r_{S}^{S'})^{-1}(D)=\bigcup_{d\in D}(r_{S}^{S'})^{-1}(d)$.}}}\smallskip{}

\noindent $\Pi$ assigns each $a\in Ag$ a \textbf{possibility correspondence
}$\Pi_{a}:\Omega_{\mathsf{F}}\rightarrow2^{\Omega_{\mathsf{F}}}$
satisfying\vspace{-4pt}
\begin{description}
\item [{\label{HMS1:confinement}Conf}] (\textbf{Confinement})\quad{}
If $w\in S'$, then $\Pi_{a}(w)\subseteq S$ for some $S\preceq S'$.
\item [{Gref\label{HMS2:G.Ref}}] (\textbf{Generalized~Reflexivity}) $w\in\left(\Pi_{a}(w)\right)^{\uparrow}$
for every $w\in\Omega_{\mathsf{F}}$.
\item [{\label{HMS3:stationarity}Stat}] (\textbf{Stationarity}) $w'\in\Pi_{a}(w)$
implies $\Pi_{a}(w')=\Pi_{a}(w)$.
\item [{\label{HMS4:PPI}PPI}] (\textbf{Projections~Preserve~Ignorance})
If $w\in S'$ and $S\preceq S'$, then\linebreak{}
 $(\Pi_{a}(w))^{\uparrow}\subseteq(\Pi_{a}(r_{S}^{S'}(w)))^{\uparrow}$.
\item [{\label{HMS5:PPK}PPK}] (\textbf{Projections~Preserve~Knowledge})
If $S\preceq S'\preceq S''$, $w\in S''$ and $\Pi_{a}(w)\subseteq S'$,
then $r_{S}^{S'}(\Pi_{a}(w))=\Pi_{a}(r_{S}^{S''}(w))$.
\end{description}
Jointly call these five properties of $\Pi_{a}$ the \textbf{\textbf{\textlabel{HMS properties}{text:HMSproperties}}}\label{text:HMSproperties}.
\end{definition}

\ref{HMS1:confinement} ensures that agents only consider possibilities
within one fixed ``vocabulary''; \ref{HMS2:G.Ref} induces factivity
of knowledge and \ref{HMS3:stationarity} yields introspection for
knowledge and awareness. \ref{HMS4:PPI} entails that at down-projected
states, agents neither ``miraculously'' know or become aware of
something new, while \ref{HMS5:PPK} implies that at down-projected
states, the agent can still ``recall'' all events she knew before,
if they are still expressible. Jointly \ref{HMS4:PPI} and \ref{HMS5:PPK}
imply that agents preserve awareness of all events at down-projected
states, if they are still expressible.
\begin{remark}
\label{rem:No-Objective-State}Unawareness frames include no objective
perspective, as agents do not\textemdash unless they are fully aware\textemdash have
a range of uncertainty defined for the maximal state-space. Taking
the maximal state-space to contain a designated `actual world' and
as providing a full and objective description of states, one can still
not evaluate agents ``true'' uncertainty/implicit knowledge. See
e.g. Figure~\ref{fig:HMSmodel} below: In $(\neg i,\ell)$, the dashed
agent's ``true'' uncertainty about $\ell$ is not determined.
\end{remark}

\subsection{Syntax-Free Unawareness}

Unawareness frames provide sufficient structure to define syntax-free
notions of knowledge and awareness. These are defined directly as
events on $\Omega_{\mathsf{F}}$.
\begin{definition}
Let $\mathsf{F}=(\mathcal{S},\preceq,\mathcal{R},\Pi)$ be an unawareness
frame. An \textbf{event} in $\mathsf{F}$ is any pair $(D^{\uparrow},S)$
with $D\subseteq S\in\mathcal{S}$ with $S$ also denoted $S(D^{\uparrow})$.
Let $\Sigma_{\mathsf{F}}$ be the set of events of $\mathsf{F}$.

The \textbf{negation} of the event $(D^{\uparrow},S)$ is $\neg(D^{\uparrow},S)=((S\backslash D)^{\uparrow},S)$.

The \textbf{conjunction} of events $\{(D_{i}^{\uparrow},S_{i})\}_{i\in I}$
is $((\bigcap_{i\in I}D_{i}^{\uparrow}),\sup_{i\in I}S_{i})$.

The events that $a$ \textbf{knows} event $(D^{\uparrow},S)$ and
where $a$ is \textbf{aware} of it are \vspace{-8pt}

{\small{}
\begin{align*}
\boldsymbol{K}_{a}((D^{\uparrow},S)) & =\begin{cases}
(\{w\in\Omega_{\mathsf{F}}\colon\Pi_{a}(w)\subseteq D^{\uparrow}\},S(D)) & \text{\,\,\,\,\,\,\,\,\,\,\,if }\exists w\in\Omega_{\mathsf{F}}.\Pi_{a}(w)\subseteq D^{\uparrow}\\
(\emptyset,S(D)) & \text{\,\,\,\,\,\,\,\,\,\,\,else}
\end{cases}\\
\boldsymbol{A}_{a}((D^{\uparrow},S)) & =\begin{cases}
(\{w\in\Omega_{\mathcal{\mathsf{F}}}\colon\Pi_{a}(w)\subseteq S(D^{\uparrow})^{\uparrow}\},S(D)) & \text{if }\exists w\in\Omega_{\mathsf{F}}.\Pi_{a}(w)\subseteq S(D^{\uparrow})^{\uparrow}\\
(\emptyset,S(D)) & \text{else}
\end{cases}
\end{align*}
}{\small\par}
\end{definition}

\noindent Negation, conjunction, knowledge and awareness events are
well-defined \cite{HMS2006,Schipper2014}. To illustrate the definitions,
some intuitions behind them: $i)$ an event modeled as a pair $(D^{\uparrow},S)$
captures that $a)$ if the event is expressible in $S$, then it is
also expressible in any $S'\succeq S$, hence $D^{\uparrow}$ is the
set of all states where the event is expressible and occurs, and $b)$
the event is expressible in the``vocabulary'' of $S$, but not the
``vocabulary'' of lower state-spaces: $D\subseteq S$ are the states
with the lowest ``vocabulary'' where the event is expressible and
occurs. \cite{Schipper2014} remarks that for $(D^{\uparrow},S)$,
if $D\ne\emptyset$, then $S$ is uniquely determined by $D^{\uparrow}$.
$ii)$ Events are given a non-binary understanding: an event $(D^{\uparrow},S)$
and it's negation does not partition $\Omega_{\mathsf{F}}$, as $s\in S'\prec S$
is in neither, but they do partition every $S''\succeq S$. $iii)$
Conjunction defined using supremum captures that the state-space required
to express the conjunction of two events is the least expressive state-space
that can express both events. $iv)$ Knowledge events are essentially
defined as in Aumann structures/state-space models: the agent knows
an event if its ``information cell'' is a subset of the event's
states. $v)$ Awareness events captures that ``\emph{an agent is
aware of an event if she considers possible states in which this event
is \textquotedblleft expressible\textquotedblright }.''\cite[p. 97]{Schipper2014}

\subsection{HMS Models}

Though unawareness frames provide a syntax-free framework adequate
for defining awareness, HMS \cite{HMS2008} use them as a semantics
for a formal language in order to identify their logic. The language
and logic are topics of Sections \ref{sec:Language} and \ref{sec:Logic}.

Instead, the models we will later define are not syntax-free. As Kripke
models, they include a valuation of atomic propositions. Therefore,
they do not correspond to unawareness frames directly, but to the
models that result by augmenting such frames with valuations. To compare
the two model classes, we define such valuations here, postponing
HMS syntax and semantics to Section~\ref{sec:Language}. Figure~\ref{fig:HMSmodel}
illustrates an HMS model, using an example inspired by \cite[p. 87]{HMS2006}

\begin{figure}[H]
\begin{centering}
\resizebox{0.8\textwidth}{!}{
		\begin{tikzpicture}
\tikzset{world/.style={rectangle, draw=black, rounded corners, text width=23pt, minimum height=18pt, 
		text centered},
	modal/.style={>=stealth',shorten >=1pt,shorten <=1pt,auto,node distance=1cm},
	state/.style={circle,draw,inner sep=0.5mm,fill=black}, 
	PCon/.style ={rectangle,draw=black, rounded corners, semithick, densely dotted, text width=30pt, minimum height=18pt, text centered},
	PCbn/.style ={draw=black, rounded corners, semithick, dashed, text width=34pt, minimum height=21pt}, 
	PCo/.style={->,densely dotted, semithick, >=stealth'},
	EPCo/.style={-,densely dotted, semithick},
	EPCb/.style={-,dashed, semithick},
	PCb/.style={->,dashed, semithick, >=stealth'}, 
	proj/.style={-, line width=0.001pt},
	block/.style ={rectangle, fill=black!10,draw=black!10, rounded corners, text width=5em, minimum height=0.3cm},
	reflexive above/.style={->,out=60,in=100,looseness=8},
	reflexive left/.style={->,out=-170,in=-190,looseness=8},
}

\begin{pgfonlayer}{foreground}
\node[world] (a) {$i,\ell$};

\node[world, right= of a] (b) {$\neg i,\ell$};
\node[world, right= of b] (c) {$  \neg i,\!\!\neg \ell$};

\node[world, below =of a, xshift=4mm, yshift=2.5mm] (b') {$\neg i$};
\node[world, left=of b'] (a') {$i$};

\node[world, right=of b', xshift=10mm] (a'') {$\ell$};
\node[world, right=of a''] (b'') {$\neg \ell$};

\node[world, below=of b, yshift=-10mm] (a''') {$\emptyset$};
\end{pgfonlayer}

\node [PCbn](g)[text width=27pt, minimum height=21pt] {}; 	

\node [PCbn, above=-7.05mm of b'] (g') [text width=27pt, minimum height=21pt] {}; 	
\node [PCon,  fit= (a)(b)(c)] (g''){}; 	
\node [PCon, fit=(a''')] (g''') {};

\draw [PCb] (a) to [out=-200,in=130,looseness=7] (g)  node {$ $};

\node[above=-3pt of a, xshift=2pt] (anchor1) {};
\draw[PCo] (a) to [out=110,in=80,looseness=8] (anchor1);

\node[above=-3pt of b, xshift=2pt] (anchor1) {};
\draw[PCo] (b) to [out=110,in=80,looseness=8] (anchor1);

\node[above=-3pt of c, xshift=2pt] (anchor1) {};
\draw[PCo] (c) to [out=110,in=80,looseness=8] (anchor1);

\draw[PCb] (b') to [out=-200,in=130,looseness=6.5] (g');

\node [PCbn, above=-7.1mm of a'''](h')[text width=27pt, minimum height=21pt] {};

\draw[PCb] (a''') to [out=135,in=105,looseness=7] (h');
\draw[PCo] (a''') to [out=90,in=60,looseness=5] (g''');

\draw [PCb] (b) to (g')  node {$ $};
\draw [PCb] (c) to (g')  node {$ $};
\draw [PCo] (b') to (g''')  node {$ $};

\node [ above=-7.1mm of b''](h'')[text width=27pt, minimum height=21pt] {}; 				
\begin{pgfonlayer}{background}
\node[block, fit=(g'')] (block) [label= right:{$S_{\{i,\ell\}}$}]{ };
\node[block, fit=(g')(a')] (block) [label=left: {$S_{\{i\}}$}]{ };

\node[block, fit=(a'')(h'')] (block) [label= right:{$S_{\{\ell\}}$}]{ };

\node[block, fit=(a''')(g''')] (block) [label= right:{$S_{\emptyset}$}]{ };
\end{pgfonlayer}       


\draw [-] (a.south) -- (a''.north)[line width=0.001pt];
\draw [-] (b.south) -- (a''.north)[line width=0.001pt];
\draw [-] (c.south) -- (b''.north)[line width=0.001pt];

\draw [-] (a.south) -- (a'.north)[line width=0.001pt];
\draw [-] (b.south) -- (b'.north)[line width=0.001pt];
\draw [-] (c.south) -- (b'.north)[line width=0.001pt];

\draw [-] (a'.south) -- (a'''.north)[line width=0.001pt];
\draw [-] (b'.south) -- (a'''.north)[line width=0.001pt];

\draw [-] (a''.south) -- (a'''.north)[line width=0.001pt];
\draw [-] (b''.south) -- (a'''.north)[line width=0.001pt];

\end{tikzpicture}
}


\vspace{-8pt}
\par\end{centering}
\caption{\label{fig:HMSmodel}An HMS model with four state-spaces (gray rectangles),
ordered spatially as a lattice. States (smallest rectangles) are labeled
with their true literals, over the set $At=\{i,\ell\}$. Thin lines
between states show projections. There are two possibility correspondences
(dashed and dotted): arrow-to-rectangle shows a mapping from state
to set (information cell). Omitted arrows go to $S_{\emptyset}$ and
are irrelevant to the story.\textbf{$\protect\phantom{aaa}$ $\protect\phantom{aaa}$Story:}
Buyer (dashed) and Owner (dotted) consider trading a firm, the price
influenced by whether $i$ (a value-raising innovation) and $\ell$
(a value-lowering lawsuit) occurs. Assume both occur and take $(i,\text{\ensuremath{\ell}})$
as actual. Then Buyer has full information, while Owner has factual
uncertainty and uncertainty about Buyer's awareness and higher-order
information, ultimately considering it possible that Buyer holds Owner
fully unaware. \emph{In detail:} Buyer's $(i,\ell)$ information cell
has both $i$ and $\ell$ defined (and is also singleton), so Buyer
is aware of them (and also knows everything). Owner is also aware
of $i$ and $\ell$, but their $(i,\ell)$ information cell contains
also $\neg i$ and $\neg\ell$ states, so Owner knows neither. Owner
is also uncertain about Buyer's information: Owner knows that either
Buyer knows $i$ and $\ell$ (cf. Buyer's $(i,\ell)$ information
cell), or Buyer knows $\neg i$, but is unaware of $\ell$ (cf. the
dashed arrows from $\neg i$ states to the less expressive state space
$S_{\{i\}}$) and then only holds it possible that Owner is unaware
of both $i$ and $\ell$ (cf. the dotted map to $S_{\emptyset}$).
See also Remark \ref{rem:HMS-redundant-states} concerning~$S_{\{\ell\}}$.}
\end{figure}
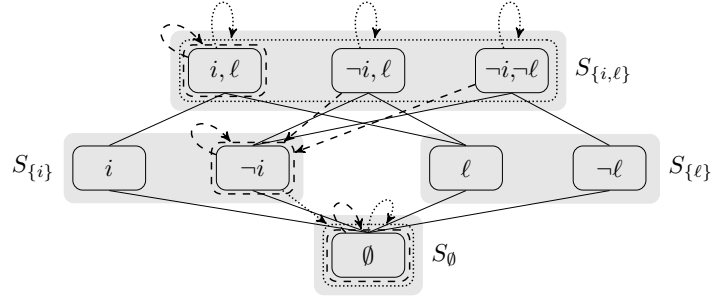

\begin{definition}
\label{def:HMS model} Let $\mathsf{F}=(\mathcal{S},\preceq,\mathcal{R},\Pi)$
be an unawareness frame with events $\Sigma_{\mathsf{F}}$. An \textbf{HMS
valuation }for $At$ and $\mathsf{F}$ is a map $V_{\mathsf{M}}:At\rightarrow\Sigma_{\mathsf{F}}$,
assigning to every atom from $At$ an event in $\mathsf{F}$. An \textbf{HMS
model} is an unawareness frame augmented with an HMS valuation, denoted
$\mathsf{M}=(\mathcal{S},\preceq,\mathcal{R},\Pi,V_{\mathsf{M}})$.
\end{definition}

\begin{remark}
\noindent \label{rem:HMS-valuations-only}HMS valuations only partially
respect the intuitive interpretation of state-spaces lattices, where
$S\preceq S'$ represents that $S'$ is at least as expressive as
$S$. If $S\preceq S'$, then $p\in At$ having defined truth value
at $S$ entails that it has defined truth value at $S'$, but if $S$
is strictly less expressive than $S'$, then this does not entail
that there is some atom $q$ with defined truth value in $S'$, but
undefined truth value in $S$. Hence, there can exist two spaces defined
for the same set of atoms, but where one is still ``strictly more
expressive'' \mbox{than the other.}
\end{remark}

\begin{remark}
\label{rem:HMS-redundant-states} Concerning Figure~\ref{fig:HMSmodel},
then the state-space $S_{\{\ell\}}$ is, in a sense, redundant: its
presence does not affect the knowledge or awareness of agents in the
state $(i,\ell)$, and it presence is not required by definition.
This stands in contrast with the corresponding Kripke lattice model
in Figure~\ref{fig:Kripke-lattice}, cf. Remark \ref{rem:KLM-redundant-states}.
\end{remark}

\section{\label{sec:Kripke-Lattice-Models}Kripke Lattice Models}

The models for awareness we construct start from Kripke models:
\begin{definition}
\label{def:kripke}A \textbf{Kripke model} for $At'\subseteq At$
is a tuple $\mathtt{K}=(W,R,V)$ where $W$ is a non-empty set of
worlds, $R:Ag\rightarrow\mathcal{P}(W^{2})$ assigns to each agent
$a\in Ag$ an accessibility relation denoted $R_{a}$, and $V:At'\rightarrow\mathcal{P}(W)$
is a valuation.

The \textbf{information cell }of $a\in Ag$ at $w\in W$ is $I_{a}(w)=\{v\in W\colon wR_{a}v\}$.
\end{definition}

\noindent The term `information cell' hints at an epistemic interpretation.
For generality, $R$ may assign non-equivalence relations. Some results
explicitly assume otherwise.

As counterpart to the HMS state-space lattice, we build a lattice
of restricted models. The below definition of the set of worlds $W_{X}$
ensures that for any $X,Y\subseteq At$, $X\ne Y$, the sets $W_{X}$
and $W_{Y}$ are disjoint, mimicking the same requirement for state-spaces.
In the restriction $\mathtt{K}_{X}$ of \textsc{$\mathtt{K}$,} it
is required that $(w_{X},v_{X})\in R_{aX}$ iff $(w,v)\in R_{a}$.
Each direction bears similarity to an HMS property: left-to-right
to \ref{HMS5:PPK} and right-to-left to \ref{HMS4:PPI}. They also
remind us, resp., of the \emph{No Miracles} and \emph{Perfect Recall}
properties from Epistemic Temporal Logic, cf. e.g., \cite{Benthem_etal_2009,vanLee_etal_2019}.
\begin{definition}
\label{def:restricted model}Let $\mathtt{K}=(W,R,V)$ be a Kripke
model for $At$. The \textbf{restriction} of $\mathtt{K}$ to $X\subseteq At$
is the Kripke model $\mathtt{K}_{X}=(W_{X},R_{X},V_{X})$ for $X$
where

$W_{X}=\{w_{X}\colon w\in W\}$ where $w_{X}$ is the ordered pair
$(w,X)$,

$R_{Xa}=\{(w_{X},v_{X})\colon(w,v)\in R_{a}\}$ and

$V_{X}:X\rightarrow\mathcal{P}(W_{X})$ such that, for all $p\in X$,$w_{X}\in V_{X}(p)$
iff $w\in V(p)$.

\noindent For the $R_{Xa}$ information cell of $a$ at $w_{X}$,
write $I_{a}(w_{X})$.
\end{definition}

To construct a lattice of restricted models, we simply order them
in accordance with subset inclusion of the atoms. This produces a
complete lattice.
\begin{definition}
\label{def:restriction lattice}Let $\mathtt{K}$ be a Kripke model
for $At$. The \textbf{restriction lattice} of $\mathtt{K}$ is $(\mathcal{K}(\mathtt{K}),\trianglelefteqslant)$
where $\mathcal{K}(\mathtt{K})=\{\mathtt{K}_{X}\}_{X\subseteq At}$
is the set of restrictions of $\mathtt{K}$, and $\mathtt{K}_{X}\trianglelefteqslant\mathtt{K}_{Y}$
iff $X\subseteq Y$.
\end{definition}

Projections in unawareness frames are informally interpreted as mapping
states to alternates of themselves in less expressive spaces. Restriction
lattices offer the same, but implemented with respect to $At$: if
$Y\subseteq X\subseteq At$, then $w_{Y}$ is the alternate of $w_{X}$
formally described by the smaller vocabulary of atoms, $Y$.

The accessibility relations of the Kripke models in a restriction
lattice account for the epistemic dimension of the HMS possibility
correspondence $\Pi_{a}$. For the awareness dimension, each agent
$a\in Ag$ is assigned an \emph{awareness map} $\pi{}_{a}$ that maps
a world $w_{X}$ down to $\pi_{a}(w_{X})=w_{Y}$ for some $Y\subseteq X$.
We think of $\pi_{a}(w_{X})$ as $a$'s \emph{awareness image} of
$w_{X}$\textemdash i.e., $w_{X}$ as it occurs to $a$ given her
(un)awareness; the submodel from $\pi_{a}(w_{X})$ is thus $a$'s
subjective perspective.

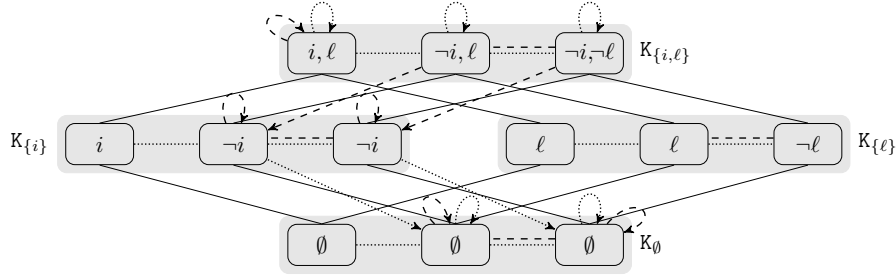
\begin{figure}
\resizebox{1\textwidth}{!}{
	\begin{tikzpicture}
\tikzset{world/.style={rectangle, draw=black, rounded corners, text width=23pt, minimum height=18pt, 
		text centered},
	modal/.style={>=stealth',shorten >=1pt,shorten <=1pt,auto,node distance=1cm},
	state/.style={circle,draw,inner sep=0.5mm,fill=black}, 
	PCon/.style ={rectangle,draw=black, rounded corners, semithick, densely dotted, text width=30pt, minimum height=18pt, text centered},
	PCbn/.style ={draw=black, rounded corners, semithick, dashed, text width=34pt, minimum height=21pt}, 
	PCo/.style={->,densely dotted, semithick, >=stealth'},
	EPCo/.style={-,densely dotted, semithick},
	EPCb/.style={-,dashed, semithick},
	PCb/.style={->,dashed, semithick, >=stealth'},
	proj/.style={-, line width=0.001pt, draw=black!2},
	block/.style ={rectangle, fill=black!10,draw=black!10, rounded corners, text width=5em, minimum height=0.3cm},
	reflexive above/.style={->,out=60,in=100,looseness=8},
	reflexive left/.style={->,out=-170,in=-190,looseness=8},
}

\begin{pgfonlayer}{foreground}
\node[world] (a) {$i,\ell$};
\node[world, right= of a] (b) {$\neg i, \ell$};
\node[world, right= of b] (c) {$\neg i,\!\!\neg \ell$};

\node[world, below =of a, xshift=7mm, yshift=2.5mm] (c') {$\neg i$};
\node[world, left=of c'] (b') {$\neg i$};
\node[world, left=of b'] (a') {$ i$};

\node[world, right=of c', xshift=6mm] (a'') {$\ell$};
\node[world, right=of a''] (b'') {$ \ell$};
\node[world, right=of b''] (c'') {$\neg \ell$};

\node[world, below=of b, yshift=-13mm] (b''') {$\emptyset$};
\node[world, right=of b'''] (c''') {$\emptyset$};
\node[world, left=of b'''] (a''') {$\emptyset$};
\end{pgfonlayer}

\node[block, fit=(a)(b)(c)] (block) [label= right:{$\mathtt{K}_{\{i,\ell\}}$}]{ };
\node[block, fit=(a')(b')(c')] (block) [label=left: {$\mathtt{K}_{\{i\}}$}]{ };

\node[block, fit=(a'')(b'')(c'')] (block) [label= right:{$\mathtt{K}_{\{\ell\}}$}]{ };

\node[block, fit=(a''')(b''')(c''')] (block) [label= right:{$\mathtt{K}_{\emptyset}$}]{ };

\draw[PCo] (a) to [out=110,in=70,looseness=7] (a);
\draw[PCo] (b) to [out=110,in=70,looseness=7] (b);
\draw[PCo] (c) to [out=110,in=70,looseness=7] (c);	
\draw[EPCb] (c) to [out=170,in=10,looseness=0] (b);	
\draw[EPCb] (c') to [out=170,in=10,looseness=0] (b');	
\draw[EPCb] (c'') to [out=170,in=10,looseness=0] (b'');	
\draw[EPCb] (c''') to [out=170,in=10,looseness=0] (b''');	
\draw[PCb] (a) to [out=-200,in=130,looseness=6] (a)  node {$ $};
\draw[PCb] (b) to  (b');	
\draw[PCb] (b') to [out=110,in=70,looseness=7] (b');
\draw[PCb] (c) to (c');	
\draw[PCb] (c') to [out=110,in=70,looseness=7] (c');	
\draw[PCo] (b') to  (b''');	
\draw[PCo] (c') to (c''');	
\draw[PCb] (b''') to [out=135,in=105,looseness=7] (b''');
\draw[PCo] (b''') to [out=90,in=55,looseness=7] (b''');
\draw[PCo] (c''') to [out=110,in=70,looseness=7] (c''');
\draw[PCb] (c''') to [out=50,in=20,looseness=6] (c''');

\draw[EPCo] (a) to (b);
\draw[EPCo] (b) to (c);
\draw[EPCo] (a') to (b');
\draw[EPCo] (b') to (c');
\draw[EPCo] (a'') to (b'');
\draw[EPCo] (b'') to (c'');
\draw[EPCo] (a''') to (b''');
\draw[EPCo] (b''') to (c''');

\draw [-] (a.south) -- (a''.north)[line width=0.001pt];
\draw [-] (b.south) -- (b''.north)[line width=0.001pt];
\draw [-] (c.south) -- (c''.north)[line width=0.001pt];

\draw [-] (a.south) -- (a'.north)[line width=0.001pt];
\draw [-] (b.south) -- (b'.north)[line width=0.001pt];
\draw [-] (c.south) -- (c'.north)[line width=0.001pt];

\draw [-] (a'.south) -- (a'''.north)[line width=0.001pt];
\draw [-] (b'.south) -- (b'''.north)[line width=0.001pt];
\draw [-] (c'.south) -- (c'''.north)[line width=0.001pt];

\draw [-] (a''.south) -- (a'''.north)[line width=0.001pt];
\draw [-] (b''.south) -- (b'''.north)[line width=0.001pt];
\draw [-] (c''.south) -- (c'''.north)[line width=0.001pt];
\end{tikzpicture}
}

\caption{\label{fig:Kripke-lattice}A Kripke lattice model of the Figure \ref{fig:HMSmodel}
example. See four restrictions (gray rectangles), ordered spatially
as a lattice. States (smallest rectangles) are labeled with their
true literals, over the set $At=\{i,\ell\}$. Horizontal dashed and
dotted lines \emph{inside restrictions} represent Buyer and Owner's
accessibility relations (omitted are links obtainable by reflexive-transitive
closure), while dotted and dashed arrows \emph{between restrictions}
represent their awareness maps (some arrows are omitted: they go to
states' alternates in $\mathtt{K}_{\emptyset}$, and are irrelevant
from $(i,l)$). Thin lines connect states with their alternate in
lower restrictions. See also Remark \ref{rem:KLM-redundant-states}
concerning~$\mathtt{K}_{\{\ell\}}$. }
\end{figure}

In the following definition, we introduce three properties of awareness
maps, which we will assume. Intuitions follow the definition.
\begin{definition}
\label{def:awareness.map}With $\mathsf{L}=(\mathcal{K}(\mathtt{K}),\trianglelefteqslant)$
a restriction lattice, let $\Omega_{\mathsf{L}}=\bigcup\mathcal{K}(\mathtt{K})$
and let $\pi$ assign to each agent $a\in Ag$ an \textbf{awareness
map} $\pi_{a}:\Omega_{\mathsf{L}}\rightarrow\Omega_{\mathsf{L}}$
satisfying
\begin{description}
\item [{\label{our1-DownwardsProjection}D}] (\textbf{Downwards}) For all
$w_{X}\in\Omega_{\mathsf{L}}$, $\pi{}_{a}(w_{X})=w_{Y}$ for some
$Y\subseteq X$.
\item [{\label{our2-Intro.Idem}I\,I}] (\textbf{Introspective Idempotence})
If $\pi_{a}(w_{X})=w_{Y}$, then for all $v_{Y}\in I_{a}(w_{Y})$,
$\pi_{a}(v_{Y})=u_{Y}$ for some $u_{Y}\in I_{a}(w_{Y})$.
\item [{\label{our3-NoSurpises}NS}] (\textbf{No Surprises}) If $\pi_{a}(w_{X})=w_{Z}$,
then for all $Y\subseteq X$, $\pi_{a}(w_{Y})=w_{Y\cap Z}$.\vspace{-3pt}
\end{description}
Call $\mathsf{K}=(\mathcal{K}(\mathtt{K}),\trianglelefteqslant,\pi)$
the \textbf{Kripke lattice model} of $\mathtt{K}$.
\end{definition}

\ref{our1-DownwardsProjection} ensures that an agent's awareness
image of a world is a restricted representation of that same world.
Hence the awareness image does not conflate worlds, and does not allow
the agent to be aware of a more expressive vocabulary than that which
describes the world she views from. With \ref{our2-Intro.Idem} and
accessibility assumed reflexive, it entails that $\pi_{a}$ is idempotent:
for all $w_{X},$ $\pi_{a}(\pi_{a}(w_{X}))=\pi_{a}(w_{X})$. Alone,
\ref{our2-Intro.Idem} states that in her awareness image, the agent
knows, and is aware of, the atoms that she is aware of. Given that
accessibility is distributed by inheritance through the Kripke models
in restriction lattices, the property implies that the same holds
for every such model. \ref{our3-NoSurpises} guarantees that awareness
remains ``consistent'' down the lattice, so that awareness of an
atom does not appear or disappear without reason. Consider the consequent
$\pi_{a}(w_{Y})=w_{Y\cap Z}$ and its two subcases $\pi_{a}(w_{Y})=w_{Y^{*}}$
with $Y^{*}\subseteq Y\cap Z$ and $Y^{*}\supseteq Y\cap Z$. Colloquially,
the first states that if atoms are removed from the description of
the world from which the agent views, then they are also removed from
her awareness. Oppositely, the second states that if atoms are removed
from the description of the world from which the agent views, then
no more than these should be removed from her awareness. Jointly,
no awareness should ``miraculously'' appear, and all \mbox{awareness
should be ``recalled''.}\footnote{Again, we are reminded of \emph{No Miracles }and \emph{Perfect Recall}.}
\begin{remark}
Contrary to HMS models (cf. Remark \ref{rem:No-Objective-State}),
Kripke lattice models have an objective perspective: designating an
`actual world' in $\mathtt{K}_{At}$ allows one to check agents'
uncertainty about the possible states of the world described by the
maximal language, i.e., from $\mathtt{K}_{At}$ we can read off their
``actual implicit knowledge''. See e.g. Figure~\ref{fig:Kripke-lattice}:
In the $(\neg i,\ell)$ state, the dashed agent's ``true'' uncertainty
about $\ell$ \emph{is }determined, contrary to the same state in
the HMS model of Figure~\ref{fig:HMSmodel}.
\end{remark}

\begin{remark}
\label{rem:KLM-redundant-states} In Remark \ref{rem:HMS-redundant-states},
we mentioned that the HMS state-space $S_{\{\ell\}}$ of Figure~\ref{fig:HMSmodel}
is redundant. Similarly, $\mathtt{K}_{\{\ell\}}$ is redundant in
Figure~\ref{fig:Kripke-lattice} (from $(i,\ell)$, $\mathtt{K}_{\{\ell\}}$
is unreachable.) However, contrary to the HMS case, it is here required
by definition, as a restriction lattice contains all restrictions
of the original Kripke model. For simplicity of constructions, we
have not here attempted to prune away redundant restrictions. A more
general model class may be obtained by letting models be based on
sub-orders of the restriction lattice. See also the concluding~remarks.
\end{remark}

\section{\label{sec:Moving}Moving between HMS Models and Kripke Lattices:
Transformations}

To clarify the relationship between HMS models and Kripke lattice
models, we introduce transformations between the two model classes,
showing that a model from one class encodes the structure of a model
from the other. The core idea is to think of a possibility correspondence
$\Pi_{a}$ as the composition of $I_{a}$ and $\pi_{a}$: $\Pi_{a}(w)$
is the information cell of the awareness image of $w$.

The propositions of this section show that the transformations produce
models of the desired class. Additionally, their proofs shed partial
light on the relationship between the HMS properties and those assumed
for awareness maps $\pi_{a}$ and accessibility relations $R_{a}$:
we discuss this shortly in the concluding remarks.

\subsection{From HMS Models to Kripke Lattice Models}

Moving from HMS models to Kripke lattice models requires a somewhat
involved construction as it must tease apart unawareness and uncertainty
from the possibility correspondences, and track the distribution of
atoms and their relationship to awareness. For an example, then the
Kripke lattice model in Figure \ref{fig:Kripke-lattice} is the HMS
model of Figure~\ref{fig:HMSmodel} transformed.
\begin{definition}
\label{def:L-transform}Let $\mathsf{M}=(\mathcal{S},\preceq,\mathcal{R},\Pi,V_{\mathsf{M}})$
be an HMS model with maximal state-space $T$. For any $O\subseteq\Omega_{\mathsf{M}}$,
let $At(O)=\{p\in At\colon O\subseteq V_{M}(p)\cup\neg V_{M}(p)\}$.\footnote{$At(O)$ contains the atoms that have a defined truth value in every
$s\in O$.}\smallskip{}

\noindent The \textbf{$L$-transform model of $\mathsf{M}$} is $L(\mathsf{M})=(\mathcal{K}(\mathtt{K}),\trianglelefteqslant,\pi)$
where the Kripke model $\mathtt{K}=(W,R,V)$ for $At$ given by\smallskip{}

\noindent \quad{}$W=T$;

\noindent \quad{}$R$ maps each $a\in Ag$ to $R_{a}\subseteq W^{2}$
s.t. $(w,v)\in R_{a}$ iff $r_{S(\Pi_{a}(w))}^{T}(v)\in\Pi_{a}(w)$;

\noindent \quad{}$V:At\rightarrow\mathcal{P}(W)$, defined by $V(p)\ni w$
iff $w\in V_{\mathsf{M}}(p)$, for every $p\in At$;\smallskip{}

\noindent $\pi$ assigns each $a\in Ag$ a map $\pi_{a}:\Omega_{L(\mathsf{M})}\rightarrow\Omega_{L(\mathsf{M})}$
s.t. for all $w_{X}\in\Omega_{L(\mathsf{M})}$,\smallskip{}

\noindent \quad{}$\pi_{a}(w_{X})=w_{Y}$ where $Y=At(S_{Y})$ for
the $S_{Y}\in\mathcal{S}$ with $S_{Y}\supseteq\Pi_{a}(r_{S_{X}}^{T}(w))$

\noindent \quad{}where $S_{X}=\min\{S\in\mathcal{S}\colon At(S)=X\}$.\smallskip{}

\noindent The \textbf{state correspondence} between $\mathsf{M}$
and $L(\mathsf{M})$ is the map $\ell:\Omega_{\mathsf{M}}\rightarrow2^{\Omega_{L(\mathsf{M})}}$
s.t. for all $s\in\Omega_{\mathsf{M}}$\smallskip{}

\noindent \quad{}$\ell(s)=\{w_{X}\in W_{X}\colon w\in(r_{S(s)}^{T})^{-1}(s)\text{ for }X=At(S(s))\}.$
\end{definition}

Intuitively, in the $L$-transform model, a world $v\in W$ is accessible
from a world $w\in W$ for an agent if, and only if, $v$'s restriction
to the agent's vocabulary at $w$ is one of the possibilities she
entertains.\footnote{We thank a reviewer of the short version of this paper \cite{BelardinelliRendsvig2020}
for this wording.} In addition, the awareness map $\pi_{a}$ of agent $a$ relates a
world $w_{X}$ to its less expressive counterpart $w_{Y}$ if, and
only if, $Y$ is the vocabulary agent $a$ adopts when describing
what she considers possible.
\begin{remark}
The $L$-transform model $L(\mathsf{M})$ of $\mathsf{M}$ is well-defined
as the object $\mathtt{K}=(W,R,V)$ is in fact a Kripke model for
$At$: $i)$ By def. of HMS models, $W=T\in\mathcal{S}$ is non-empty;
$ii)$ for each $a$, $R_{a}\subseteq W^{2}$ is well-defined: if
$w\in T=W$, then by \ref{HMS1:confinement}, $\Pi_{a}(w)\subseteq S$,
for some $S\in\mathcal{S}$. Hence, $U=\{v\in T\colon r_{S}^{T}(v)\in\Pi_{a}(w)\}$
is well-defined, and so is $\{(w,v)\in T^{2}\colon v\in U\}=R_{a}$;
$iii)$ As $V_{\mathsf{M}}$ is an HMS valuation $V_{\mathsf{M}}:At\rightarrow\Sigma$
for $At$, clearly $V$ is valuation for $At$. Hence $\mathtt{K}=(W,R,V)$
is a Kripke model for $At$.
\end{remark}

\begin{remark}
The $\min$ used in defining $S_{X}$ is due to the issue of Remark
\ref{rem:HMS-valuations-only}.
\end{remark}

\begin{remark}
\label{rem:state-correspondence}The state correspondence map $\ell$
is also well-defined. That it maps each state in $\Omega_{\mathsf{M}}$
to a \emph{set }of worlds in $\Omega_{L\mathsf{(M)}}$ points to a
construction difference between HMS models and Kripke lattice models:
in the former, the downwards projections of two states may `merge'
them, so state-spaces may shrink when moving down the lattice; in
the latter, distinct worlds remain distinct, so all world sets in
a restriction lattice share cardinality.
\end{remark}

As unawareness and uncertainty are separated in Kripke lattice models,
we show two results about $L$-transforms. The first shows that the
\ref{HMS1:confinement}, \ref{HMS3:stationarity} and \ref{HMS5:PPK}
entail that $\pi_{a}$ assigns awareness maps, and the second that
the five HMS properties entail that $R$ assigns equivalence relations.
In showing the first, we make use of the following lemma, which intuitively
shows that the information cell of an agent contains a state described
with a certain vocabulary if, and only if, the agent considers possible
the corresponding state described with the same vocabulary:
\begin{lemma}
\label{LemmaYYY} For\emph{ }every $w_{Y}\in\Omega_{\mathsf{K}}$,
if $\Pi_{a}(w)\subseteq S$ and $At(S)=Y$, then $v_{Y}\in I_{a}(w_{Y})$
iff $v_{S}\in\Pi_{a}(w)$.
\end{lemma}

\begin{proof}
Let $w_{Y}\in\Omega_{L\mathsf{(M)}}$. Consider the respective $w\in T=W$
and let $\Pi_{a}(w)\subseteq S$, with $At(S)=Y$. Assume that $v_{Y}\in I_{a}(w_{Y})$.
This is the case iff (def. of $I_{a}$) $(w_{Y},v_{Y})\in R_{Ya}$
iff (def. of restriction lattice) $(w,v)\in R_{a}$ iff (Def. \ref{def:L-transform})
$v_{S}\in\Pi_{a}(w)$.
\end{proof}

\begin{proposition}
\label{prop:L-transform-is-KLM}For any HMS model $\mathsf{M}$, its
$L$-transform $L(\mathsf{M})$ is a Kripke lattice model.
\end{proposition}

\begin{proof}
Let $\mathsf{M}=(\mathcal{S},\preceq,\mathcal{R},\Pi,V_{\mathsf{M}})$
be an HMS model with maximal state-space $T$. We show that $L(\mathsf{M})=(\mathcal{K}(\mathtt{K}),\trianglelefteqslant,\pi)$
is a Kripke lattice model by showing that $\pi_{a}$ satisfies the
three properties of an awareness map:

\emph{\ref{our1-DownwardsProjection}}: Consider an arbitrary $w_{X}\in\Omega_{L(\mathsf{M})}$.
By def. of $L$-transform, $X=At(S)$ for some $S\in\mathcal{S}$.
Let $S_{X}=\min\{S\in\mathcal{S}\colon At(S)=X\}$. If $w_{X}\in W_{X}$
then for some $w\in W=T$, $w_{S_{X}}\in S_{X}$. By \ref{HMS1:confinement},
$\Pi_{a}(w_{S_{X}})\subseteq S_{Y}$, for some $S_{Y}\preceq S_{X}$.
Let $Y=At(S_{Y})$. Then, by def. of $\pi_{a}$, $\pi_{a}(w_{X})=w_{Y}$
and $Y\subseteq X$.

\emph{\ref{our2-Intro.Idem}}: Let $\pi_{a}(w_{X})=w_{Y}$. By def.
of $\pi_{a}$, it holds that $\Pi_{a}(r_{S_{X}}^{T}(w))\subseteq S_{Y}$
with $At(S_{Y})=Y$ and $S_{X}=\min\{S\in\mathcal{S}\colon At(S)=X\}$.
For a contradiction, suppose there exists a $v_{Y}\in I_{a}(w_{Y})$
such that for all $u_{Y}\in I_{a}(w_{Y})$, $\pi_{a}(v_{Y})\not=u_{Y}$.
Then $\pi_{a}(v_{Y})=t_{Z}$ for some $Z\subseteq Y$ and $t_{Z}\not\in I_{a}(w_{Y})$.
By def. of $\pi_{a}$, $\pi_{a}(v_{Y})=t_{Z}$ iff $\Pi_{a}(r_{S_{Y}}^{T}(v))\subseteq S_{Z}$,
where $Z=At(S_{Z})$. Then, by Lemma \ref{LemmaYYY}, $t_{Z}\in I_{a}(v_{Z})$
iff $t_{S_{Z}}\in\Pi_{a}(r_{S_{Y}}^{T}(v))$. Moreover, as $\Pi_{a}(r_{S_{X}}^{T}(w))\subseteq S_{Y}$
and $At(S_{X})=X$, by Lemma \ref{LemmaYYY}, it also follows that
$v_{Y}\in I_{a}(w_{Y})$ iff $v_{S_{Y}}\in\Pi_{a}(r_{S_{X}}^{T}(w))$.
Since $v_{Y}\in I_{a}(w_{Y})$ then $v_{S_{Y}}\in\Pi_{a}(r_{S_{X}}^{T}(w))$.
Hence, by \ref{HMS3:stationarity}, $\Pi_{a}(r_{S_{X}}^{T}(w))=\Pi_{a}(r_{S_{Y}}^{T}(v))$,
which implies $t_{S_{Z}}\in\Pi_{a}(r_{S_{X}}^{T}(w))$. But then $t_{Z}\in I_{a}(v_{Z})$,
contradicting the assumption that $t_{Z}\not\in I_{a}(w_{Y})$. Thus,
for all $v_{Y}\in I_{a}(w_{Y})$, $\pi_{a}(v_{Y})=u_{Y}$ for some
$u_{Y}\in I_{a}(w_{Y})$.

\emph{\ref{our3-NoSurpises}}: Let $\pi_{a}(w_{X})=w_{Y}$. By \ref{our1-DownwardsProjection}
(cf. item 1. above), $Y\subseteq X$. Consider an arbitrary $Z\subseteq X$.
We have two cases: either $i)$ $Z\subseteq Y$ or $ii)$ $Y\subseteq Z$.
$i)$: then $Z\subseteq Y\subseteq X$. Let $Z=At(S_{Z})$, $Y=At(S_{Y})$,
and $X=At(S_{X})$. Then $S_{Z}\preceq S_{Y}\preceq S_{X}$. By \ref{HMS5:PPK},
$\big(\Pi_{a}(r_{S_{X}}^{T}(w))\big)_{Z}=\Pi_{a}(r_{S_{Z}}^{T}(w))$.
As $\pi_{a}(w_{X})=w_{Y}$, by def. of $\pi_{a}$, $\Pi_{a}(r_{S_{X}}^{T}(w))\subseteq S_{Y}$.
Then $\big(\Pi_{a}(r_{S_{X}}^{T}(w))\big)_{Z}=r_{S_{Z}}^{S_{Y}}\big(\Pi_{a}(r_{S_{X}}^{T}(w))\big)\subseteq S_{Z}$.
Hence $\Pi_{a}(r_{S_{Z}}^{T}(w))\subseteq S_{Z}$, and by def. of
$\pi_{a}$, $\pi_{a}(w_{Z})=w_{Z}$. As $Z\subseteq Y$, $\pi_{a}(w_{Z})=w_{Z}=w_{Z\cap Y}$.
$ii)$: then $Y\subseteq Z\subseteq X$. By analogous reasoning, we
have $\pi_{a}(w_{Y})=w_{Y}=w_{Y\cap Z}$ as $Y\subseteq Z$. We can
conclude that if $\pi_{a}(w_{X})=w_{Y}$, then for all $Z\subseteq X$,
$\pi_{a}(w_{Z})=w_{Z\cap X}$.
\end{proof}

\begin{proposition}
\label{prop:L-transform-has-EQ-R}If $L(\mathsf{M})=(\mathcal{K}(\mathtt{K}=(W,R,V)),\trianglelefteqslant,\pi)$
is the $L$-transform of an HMS model $\mathsf{M}$, then for every
$a\in Ag$, $R_{a}$ is an equivalence relation.
\end{proposition}

\begin{proof}
\noindent Let $\mathsf{M}=(\mathcal{S},\preceq,\mathcal{R},\Pi,V_{\mathsf{M}})$
have maximal state-space~$T$.

\emph{Reflexivity:} Let $w\in T$ and $\Pi_{a}(w)\subseteq S$, for
some $S\in\mathcal{S}$. By def. of upwards closure, $(\Pi_{a}(w))^{\uparrow}=\bigcup_{S'\succeq S}(r_{S}^{S'})^{-1}(\Pi_{a}(w))$,
and by \ref{HMS2:G.Ref}, $w\in(\Pi_{a}(w))^{\uparrow}=\bigcup_{S'\succeq S}(r_{S}^{S'})^{-1}(\Pi_{a}(w))$.
Since $T\succeq S$, then $r_{S}^{T}(w)\in\Pi_{a}(w)$. Thus, $(w,w)\in R_{a}$,
by def. $L$-transform. By def. of restriction lattices, this holds
for all $A\subseteq At$, i.e. $(w_{A},w_{A})\in R_{Aa}$.

\emph{Transitivity:} Let $w,v,u$ be in $T$. By \ref{HMS1:confinement},
there are $S,S'\in\mathcal{S}$ such that $\Pi_{a}(w)\subseteq S$
and $\Pi_{a}(v)\subseteq S'$. Assume that $(w,v)\in R_{a}$ and $(v,u)\in R_{a}$.
By def. of $R_{a}$, then $r_{S}^{T}(v)\in\Pi_{a}(w)$ and $r_{S'}^{T}(u)\in\Pi_{a}(v)$.
By \ref{HMS3:stationarity}, $\Pi_{a}(w)=\Pi_{a}(r_{S}^{T}(v))$ and
$\Pi_{a}(v)=\Pi_{a}(r_{S'}^{T}(u))$. As $v\in T$ and $S\preceq T$,
by \ref{HMS4:PPI}, $\Pi_{a}(v)^{\uparrow}\subseteq\Pi_{a}(r_{S}^{T}(v))^{\uparrow}=\Pi_{a}(w)^{\uparrow}$.
Hence, as $r_{S'}^{T}(u)\in\Pi_{a}(v)^{\uparrow}$, also $r_{S'}^{T}(u)\in\Pi_{a}(w)^{\uparrow}$.
By def. of upwards closure, $r_{S}^{T}(u)\in\Pi_{a}(w)$. Finally,
$(w,u)\in R_{a}$ by def. of $R_{a}$.

\emph{Symmetry:} Let $w,v\in T$ be in $T$. Assume that $(w,v)\in R_{a}$.
By \ref{HMS1:confinement}, there are $S,S'\in\mathcal{S}$ such that
$\Pi_{a}(w)\subseteq S$ and $\Pi_{a}(v)\subseteq S'$. Then $r_{S}^{T}(v)\in\Pi_{a}(w)$
(def. of $L$-transform), and by \ref{HMS3:stationarity}, $\Pi_{a}(w)=\Pi_{a}(r_{S}^{T}(v))$.
As $v\in T$ and $T\succeq S$, by \ref{HMS4:PPI}, by $\Pi_{a}(v)^{\uparrow}\subseteq\Pi_{a}(r_{S}^{T}(v))^{\uparrow}$.
Then, by def. of upwards closure, $T\text{\ensuremath{\succeq S'\succeq S}}$.
As $v\in T$, by \ref{HMS5:PPK}, $r_{S}^{S'}(\Pi_{a}(v))=\Pi_{a}(r_{S}^{T}(v))$.
By \ref{HMS2:G.Ref}, $x\in\Pi_{a}(w)^{\uparrow}$, and since $\Pi_{a}(w)\subseteq S$
then $r_{S}^{T}(w)\in\Pi_{a}(w)$, by def. of upward closure. Then
$r_{S}^{T}(w)\in\Pi_{a}(w)=\Pi_{a}(r_{S}^{T}(v))=r_{S}^{S'}(\Pi_{a}(v))$.
So $r_{S}^{T}(w)\in r_{S}^{S'}(\Pi_{a}(v))$, i.e. $r_{S'}^{T}(w)\in\Pi_{a}(v)$,
by def. of $r$. Hence, $(v,w)\in R_{a}$, by def. of $R_{a}$.
\end{proof}

\subsection{From Kripke Lattice Models to HMS Models}

Moving from Kripke lattice models to HMS models requires a less involved
construction, as the restriction lattice almost encodes projections,
and unawareness and uncertainty are simply composed to form possibility
correspondences:
\begin{definition}
\label{def:H-transform}Let $\mathsf{K}=(\mathcal{K}(\mathtt{K}=(W,R,V)),\trianglelefteqslant,\pi)$
be a Kripke lattice model for $At$. The \textbf{$H$-transform} of
\textbf{$\mathsf{K}$} is $H(\mathsf{K})=(\mathcal{S},\preceq,\mathcal{R},\Pi,V_{H(\mathsf{K})})$
where

\noindent $\mathcal{S}=\{W_{X}\subseteq\Omega_{\mathsf{K}}:\mathtt{K}_{X}\in\mathcal{K}(\mathtt{K})\}$;

\noindent $W_{X}\preceq W_{Y}$ iff $\mathtt{K}_{X}\trianglelefteqslant\mathtt{K}_{Y}$;

\noindent $\mathcal{R}=\{r_{W_{Y}}^{W_{X}}\colon r_{W_{Y}}^{W_{X}}(w_{X})=w_{Y}\text{ for all }w\in W,\text{ and all }X,Y\subseteq At\}$;

\noindent $\Pi=\{\Pi_{a}\in(2^{\Omega_{\mathsf{K}}})^{\Omega_{\mathsf{K}}}\colon\Pi_{a}(w_{X})=I_{a}(\pi_{a}(w_{X}))\text{ for all }w\in W,X\subseteq At,a\in Ag\}$;

\noindent $V_{H(\mathsf{K})}(p)=\{w_{X}\in\Omega_{\mathsf{K}}\colon X\ni p\text{ and }w_{X}\in V_{X}(p)\}$
for all $p\in At$.
\end{definition}

As HMS models lump together unawareness and uncertainty, we show only
one result in this direction:
\begin{proposition}
\label{prop:H-transform-is-HMS}For any Kripke lattice model $\mathsf{K}=(\mathcal{K}(\mathtt{K}=(W,R,V)),\trianglelefteqslant,\pi)$
such that $R$ assigns equivalence relations, the $H$-transform $H(\mathsf{K})$
is an HMS~model.
\end{proposition}

\begin{proof}
Let $\mathsf{K}$ be as stated and let $H(\mathsf{K})=(\mathcal{S},\preceq,\mathcal{R},\Pi,V_{H(\mathsf{K})})$
be its $H$-transform.

\noindent $\mathcal{S}=\{W_{X},W_{Y},...\}$ is composed of non-empty
disjoint sets by construction and $(\mathcal{S},\preceq)$ is a complete
lattice as $(\mathcal{K}(\mathtt{K}),\trianglelefteqslant)$ is so.
$\mathcal{R}$ is clearly a family of well-defined, surjective and
commutative projections. As $\Pi$ assigns to each $a\in Ag$, $\Pi_{a}(w_{X})=I_{a}(\pi_{a}(w_{X}))$,
for all $w\in W$, $X\subseteq At$, it assigns $a$ a map $\Pi_{a}:\Omega_{H(\mathsf{K})}\rightarrow2^{\Omega_{H(\mathsf{K})}}$,
which is a possibility correspondence as it satisfies the HMS properties:

\emph{\ref{HMS1:confinement}}: For $w_{X}\in W_{X}$, $\Pi_{a}(w_{X})=I_{a}(\pi_{a}(w_{X}))$,
by Def. \ref{def:H-transform}. By \ref{our1-DownwardsProjection},
$\pi_{a}(w_{X})=w_{Y}$ for some $Y\subseteq X$, and $I_{a}(\pi_{a}(w_{X}))=I_{a}(w_{Y})$.
So, $\Pi_{a}(w_{X})\subseteq W_{Y}$ for some $Y\subseteq X$.

\emph{\ref{HMS2:G.Ref}}:\textbf{ }Let $w_{X}\in\Omega_{\mathsf{K}}$,
$X\subseteq At.$ By \ref{our1-DownwardsProjection}, $\pi_{a}(w_{X})=w_{Y}$
for some $Y\subseteq X$. By def. of $\Pi_{a}$ and $I_{a}$, $\Pi_{a}(w_{X})=I_{a}(w_{Y})=\{v_{Y}\in\Omega_{\mathsf{K}}:(w_{Y},v_{Y})\in R_{Ya}\}$.
Hence $\Pi_{a}(w_{X})\subseteq W_{Y}$. By def. of upward closure,
$(\Pi_{a}(w_{X}))^{\uparrow}=(I_{a}(w_{Y})){}^{\uparrow}=\{u_{Z}\in\Omega_{\mathsf{K}}:Y\subseteq Z\text{ and }u_{Y}\in\{v_{Y}\in\Omega_{\mathsf{K}}:(w_{Y},v_{Y})\in R_{Ya}\}\}$,
with the last identity given by the def. of $r_{W_{Y}}^{W_{Z}}$.
As $R_{a}$ is an equivalence relation, so is $R_{Ya}$, by def. So
$w_{Y}\in\{v_{Y}\in\Omega_{\mathsf{K}}:(w_{Y},v_{Y})\in R_{Ya}\}$,
and since $Y\subseteq X$, then \textbf{$w_{X}\in(\Pi_{a}(w_{X})){}^{\uparrow}$}.

\emph{\ref{HMS3:stationarity}}:\textbf{ }For $w_{X}\in\Omega_{\mathsf{K}}$,
assume $v\in\Pi_{a}(w_{X})=I_{a}(\pi_{a}(w_{X}))$. By \ref{our1-DownwardsProjection},
$v\in I_{a}(w_{Y})$, for some $Y\subseteq X$. With $R_{Ya}$ an
equivalence relation, $v\in I_{a}(w_{Y})$ iff $w_{Y}\in I_{a}(v)$,
i.e., $I_{a}(v)=I_{a}(w_{B})$. \ref{our2-Intro.Idem} and \ref{our1-DownwardsProjection}
entails that for all $u_{Y}\in I_{a}(w{}_{Y})$, $\pi_{a}(u_{Y})=u_{Y}$,
so $\pi_{a}(v)=v$. Therefore $\Pi_{a}(v)=I_{a}(\pi_{a}(v))=I_{a}(v)=I_{a}(w_{Y})=I_{a}(\pi_{a}(w_{X}))=\Pi_{a}(w_{X})$.
Thus, if $v\in\Pi_{a}(w_{X})$, then $\Pi_{a}(v)=\Pi_{a}(w_{X})$.

\emph{\ref{HMS4:PPI}}:\textbf{ }Let $w_{X}\in W_{X}$ and $W_{Y}\preceq W_{X}$,
i.e. $Y\subseteq X\subseteq At$. Let $q_{Q}\in(\Pi_{a}(w_{X}))^{\uparrow}$
with $Q\subseteq At$. By def. of $\Pi_{a}$ and \ref{our1-DownwardsProjection},
$\Pi_{a}(w_{X})=I_{a}(\pi_{a}(w_{X}))=I_{a}(w_{Z})$ for some $Z\subseteq X$.
By def. of upwards closure, it follows that $q_{Z}\in I_{a}(w_{Z})=\Pi_{a}(w_{X})$.
Now let $\pi_{a}(w_{Y})=w_{P}$ for some $P\subseteq Y$. Then, by
\ref{our3-NoSurpises}, $P=Z\cap Y$, so $P\subseteq Z$. As $q_{Z}\in I_{a}(w_{Z})$,
then $q_{P}\in I_{a}(w_{P})=I_{a}(\pi_{a}(w_{Y}))=\Pi_{a}(w_{Y})$,
by def. of restriction lattice. Since $q_{Q}\in(\Pi_{a}(w_{X}))^{\uparrow}=(I_{a}(w_{Z}))^{\uparrow}$,
then $Z\subseteq Q$. It follows that $P\subseteq Z\subseteq Q$,
which implies $q_{Q}\in(\Pi_{a}(w_{Y}))^{\uparrow}$. Hence, if $q_{Q}\in(\Pi_{a}(w_{X}))^{\uparrow}$,
then $q_{Q}\in(\Pi_{a}(w_{Y}))^{\uparrow}$, i.e., $(\Pi_{a}(w_{X}))^{\uparrow}\subseteq(\Pi_{a}(w_{Y}))^{\uparrow}$.

\emph{\ref{HMS5:PPK}}:\textbf{ }Suppose that $W_{Z}\preceq W_{Y}\preceq W_{X}$,
$w_{X}\in W_{X}$ and $\Pi_{a}(w_{X})\subseteq W_{Y}$, i.e. $\Pi_{a}(w_{X})=I_{a}(w_{Y})$
and $\pi_{a}(w_{X})=w_{Y}$. As $Z\subseteq Y\subseteq X$, \ref{our3-NoSurpises}
implies $\pi_{a}(w_{Z})=w_{Z\cap Y}=w_{Z}$. Hence, $\Pi_{a}(w_{Z})=I_{a}(w_{Z})\subseteq W_{Z}$.
Hence \ref{HMS5:PPK} is established if $\left(I_{a}(w_{Y})\right)_{Z}=I_{a}(w_{Z})$.
As $\left(I_{a}(w_{Y})\right)_{Z}=\{x_{Z}\in\Omega_{\mathsf{K}}:x_{Y}\in I_{a}(w_{Y})\}$,
then clearly $\left(I_{a}(w_{Y})\right)_{Z}=I_{a}(w_{Z})$. Thus,
$\left(\Pi_{a}(w_{X})\right)_{Z}=\Pi_{a}(w_{Z})$.

Finally, $V_{H(\mathsf{K})}$ is an HMS valuation as for each $p\in At$,
$V_{H(\mathsf{K})}(p)$ is an event $(D^{\uparrow},S)$ with $D=\{w_{\{p\}}\in W_{\{p\}}:w_{\{p\}}\in V_{\{p\}}(p)\}$
and $S=W_{\{p\}}$.\vspace{-3pt}
\end{proof}

\section{\label{sec:Language}Language for Awareness and $\mathcal{L}$-Equivalence}

Multiple languages for knowledge and awareness exist. The Logic of
General Awareness ($\Lambda_{LGA}$, \cite{HalpernFagin88}) which
we will see in Section \ref{sec:L^A}, takes implicit knowledge and
awareness as primitives, and define explicit knowledge as `implicit
knowledge $\wedge$ awareness'; other combinations are discussed
in \cite{Velazquez-Quesada2010-BENTDO-8}. 

HMS \cite{HMS2008} follow instead Modica-Rustichini \cite{ModicaRustichini1994,ModicaRustichini1999}
and take explicit knowledge as primitive and awareness as defined:
an agent is aware of $\varphi$ iff she either explicitly knows $\varphi$,
or explicitly knows that she does not explicitly know~$\varphi$.
\begin{definition}
\label{prop:HMS language}With $a\in Ag$ and $p\in At,$ define the
language $\mathcal{L}$ by
\[
\varphi::=\top\mid p\mid\neg\varphi\mid\varphi\wedge\varphi\mid K_{a}\varphi
\]
and define $A_{a}\varphi:=K_{a}\varphi\vee K_{a}\neg K_{a}\varphi$.

Let $At(\varphi)=\{p\in At\colon p\text{ is a subformula of }\varphi\}$,
for all $\varphi\in\mathcal{L}$.
\end{definition}

\subsection{HMS Models as a Semantics for $\mathcal{L}$}

The satisfaction of formulas over HMS models is defined as follows.
The semantics are three-valued, so formulas may have undefined truth
value: there may exist a $w\in\Omega_{\mathsf{M}}$ such that neither
$\mathsf{M},w\vDash\varphi$ nor $\mathsf{M},w\vDash\neg\varphi$.
This happens if and only if $\varphi$ contains atoms with undefined
truth value in $w$.
\begin{definition}
Let $\mathsf{M}=(\mathcal{S},\preceq,\mathcal{R},\Pi,V_{\mathsf{M}})$
be an HMS model and let $w\in\Omega_{\mathsf{M}}$. Satisfaction of
\textup{$\mathcal{L}$} formulas is given by
\noindent \begin{center}
\begin{tabular}{lccllllll}
$\mathsf{M},w\vDash\top$ &  & \multicolumn{2}{c}{for all $w\in\Omega_{\mathsf{M}}$} &  &  &  &  & \tabularnewline
$\mathsf{M},w\vDash p$ &  & iff\enskip{} & $w\in V_{\mathsf{M}}(p)$ & \qquad{}\qquad{} & $\mathsf{M},w\vDash\varphi\wedge\psi$ &  & iff\enskip{} & $w\in\llbracket\varphi\rrbracket\cap\llbracket\psi\rrbracket$\tabularnewline
$\mathsf{M},w\vDash\neg\varphi$ &  & iff\enskip{} & $w\in\neg\llbracket\varphi\rrbracket$ &  & $\mathsf{M},w\vDash K_{a}\varphi$ &  & iff\enskip{} & $w\in\boldsymbol{K}_{a}(\llbracket\varphi\rrbracket)$\tabularnewline
\end{tabular}
\par\end{center}

\noindent where $\llbracket\varphi\rrbracket=\{v\in\Omega_{\mathsf{M}}\colon\mathsf{M},v\vDash\varphi\}$
for all $\varphi\in\mathcal{L}$.
\end{definition}

With the HMS semantics being three-valued, they adopt a non-standard
notion of validity which requires only that a formula be always satisfied
\emph{if its has a defined truth value}. The below is equivalent to
the definition in \cite{HMS2008}, but is stated so that it also works
for Kripke lattice models:
\begin{definition}
A formula $\varphi\in\mathcal{L}$ is valid over a class of models
$\boldsymbol{C}$ iff for all models $M\in\boldsymbol{C}$, for all
states $w$ of $M$ which satisfy $p$ or $\neg p$ for all $p\in At(\varphi)$,
$w$ also satisfies $\varphi$.
\end{definition}

\subsection{Kripke Lattice Models as a Semantics for $\mathcal{L}$}

We define semantics for $\mathcal{L}$ over Kripke lattice models.
Like the HMS semantics, the semantics are three-valued, as it is possible
that a pointed Kripke lattice model $(M,w_{X})$ satisfies neither
$\varphi$ nor $\neg\varphi$. This happens exactly when $\varphi$
contains atoms not in $X$.
\begin{definition}
\label{def:Kripke lattice semantics for L}Let\textbf{ }$\mathsf{K}=(\mathcal{K}(\mathtt{K}=(W,R,V)),\trianglelefteqslant,\pi)$
be a Kripke lattice model with $w_{X}\in\Omega_{\mathsf{K}}$. Satisfaction
of \textup{$\mathcal{L}$} formulas is given by

\noindent %
\begin{tabular}{lcc>{\raggedright}p{0.44\textwidth}>{\raggedright}p{0.24\textwidth}}
$\mathsf{K},w_{X}\Vdash\top$ &  &  & for all $w_{X}\in\Omega_{\mathsf{K}}$ & \tabularnewline
$\mathsf{K},w_{X}\Vdash p$  &  & iff\enskip{} & $w_{X}\in V_{X}(p)$ & and $p\in X$\tabularnewline
$\mathsf{K},w_{X}\Vdash\neg\varphi$  &  & iff\enskip{} & not $\mathsf{K},w_{X}\Vdash\varphi$ & and $At(\varphi)\subseteq X$\tabularnewline
$\mathsf{K},w_{X}\Vdash\varphi\wedge\psi$  &  & iff\enskip{} & $\mathsf{K},w_{X}\Vdash\varphi$ and $\mathsf{K},w_{X}\Vdash\psi$\quad{} & and $At(\varphi\wedge\psi)\subseteq X$\tabularnewline
$\mathsf{K},w_{X}\Vdash K_{a}\varphi$ &  & iff\enskip{} & $\pi_{a}(w_{X})R_{Ya}v_{Y}$ implies $\mathsf{K},v_{Y}\Vdash\varphi$, 

for $Y\subseteq At$ s.t. $\pi_{a}(w_{X})\in W_{Y}$ & $\phantom{.}$

and $At(\varphi)\subseteq X$\tabularnewline
\end{tabular}
\end{definition}

\subsection{The $\mathcal{L}$-Equivalence of HMS and Kripke Lattice Models}

$L$- and $H$-transforms not only produce models of the correct class,
but also preserve finer details, as any model and its transform satisfy
the same $\mathcal{L}$ formulas.
\begin{proposition}
\label{prop:equivalence}For any HMS model $\mathsf{M}$ with $L$-transform
$L(\mathsf{M})$, for all $\varphi\in\mathcal{L}$, for all $w\in\Omega_{\mathsf{M}}$,
and for all $v\in\ell(w)$, 
\[
\mathsf{M},w\vDash\varphi\text{ iff }L(\mathsf{M}),v\Vdash\varphi.
\]
\end{proposition}

\begin{proof}
Let $\Sigma_{\mathsf{M}}$ be the events of $\mathsf{M}=(\mathcal{S},\preceq,\mathcal{R},\Pi,V_{\mathsf{M}})$
with maximal state-space $T$ and let $L(\mathsf{M})=(\mathcal{K}(\mathtt{K}=(W,R,V)),\trianglelefteqslant,\pi)$.
The proof is by induction on formula complexity. Let $\varphi\in\mathcal{L}$
and let $w\in\Omega_{\mathsf{M}}$ with $At(S(w))=X$.

\emph{Base:} $i)$ $\varphi:=p\in At$ or $ii)$ $\varphi:=\top$.
$i)$ $\mathsf{M},w\vDash p$ iff $w\in V_{\mathsf{M}}(p)$. As $V_{\mathsf{M}}(p)\in\Sigma_{\mathsf{M}}$,
$(r_{S(w)}^{T})^{-1}(w)\subseteq V_{\mathsf{M}}(p)$. By def. of $L(\mathsf{M})$,
if $v\in T=W$, then $v\in V_{\mathsf{M}}(p)$ iff $v\in V(p)$, so
$v\in(r_{S(w)}^{T})^{-1}(w)$ iff $v\in V(p)$ iff $v_{X}\in V_{X}(p)$,
with $p\in X$ (def. of Kripke lattice models). Hence, by def. of
$\ell$, $v\in\ell(w)=\{u_{X}\in W_{X}\colon u\in(r_{S(w)}^{T})^{-1}(w)\text{ for }X=At(S(w))\}$
iff $v\in V_{X}(p)$, i.e., iff $L(M),v\Vdash p$ for all $v\in\ell(w)$.
$ii)$ is trivial.

\emph{Step.} Assume $\psi,\chi\in\mathcal{L}$ satisfy Prop. \ref{prop:equivalence}.

$\varphi:=\neg\psi$. There are two cases: $i)$ $At(\psi)\subseteq At(S(w))$
or $ii)$ $At(\psi)\not\subseteq At(S(w))$. $i)$ $\mathsf{M},w\vDash\neg\psi$
iff (def. of $\vDash$) $w\in\neg\llbracket\psi\rrbracket$ iff (def.
of $V_{\mathsf{M}}$) $(r_{S(w)}^{T})^{-1}(w)\subseteq\neg\llbracket\psi\rrbracket$
iff (def. of $\llbracket\psi\rrbracket$) for all $v\in(r_{S(w)}^{T})^{-1}(w)$,
$\mathsf{M},v\not\vDash\psi$ iff (Def. \ref{def:L-transform}) for
all $v\in(r_{S(w)}^{T})^{-1}(w)$, not $L(\mathsf{M}),v\Vdash\psi$
iff (def. of $\ell(w)$) for all $v_{X}\in\ell(w)$, not $L(\mathsf{M}),v_{X}\Vdash\psi$,
with $At(\psi)\subseteq X$ iff (def. of $\Vdash$) for all $v_{X}\in\ell(w)$,
$L(\mathsf{M}),v_{X}\Vdash\neg\psi$. $ii)$ is trivial: $\varphi$
is undefined in $(\mathsf{M},w)$ iff it is so in $(L(\mathsf{M}),w_{X})$.

$\varphi:=\psi\wedge\chi$. The case follows by tracing \emph{iff}s
through the definitions of $\vDash$, $V_{\mathsf{M}}$, $\llbracket\cdot\rrbracket$,
$\big(r_{S(w)}^{T}\big)^{-1}$, $L$-transform, $\ell$, and $\Vdash$.

$\varphi:=K_{a}\psi$. $\mathsf{M},w\vDash K_{a}\psi$ iff (def. of
$\vDash$) $w\in\boldsymbol{K}_{a}(\llbracket\psi\rrbracket)$ iff
(def. of $\boldsymbol{K}_{a}$) $\Pi_{a}(w)\subseteq\llbracket\psi\rrbracket$.
Let $\Pi_{a}(w)\subseteq S$, for some $S\in\mathcal{S}$, and let
$X=At(S(w))$ and $Y=At(S)$. Then $v_{S}\in\Pi_{a}(w)\subseteq\llbracket\psi\rrbracket$
iff (def. of $\llbracket\psi\rrbracket$) for all $v_{S}\in\Pi_{a}(w)$,
$\mathsf{M},v_{S}\vDash\psi$ iff (def. of $V_{\mathsf{M}}$) for
all $(r_{S}^{T})^{-1}(v_{S})$ with $v_{S}\in\Pi_{a}(w)$, $\mathsf{M},v_{T}\vDash\psi$
iff (def. of $L$-transform) for all $v_{At}$ with $r_{S}^{T}(v)\in\Pi_{a}(w)$,
$L(\mathsf{M}),v_{At}\Vdash\psi$ and $At(\psi)\subseteq At$ iff
(def. of $L$-transform) for all $v_{At}$ with $(w_{At},v_{At})\in R_{Ata}$,
$L(\mathsf{M}),v_{At}\Vdash\psi$ and $At(\psi)\subseteq At$ iff
(def. of restriction lattice) for all $v_{Y}$ with $(w_{Y},v_{Y})\in R_{Ya}$,
$L(\mathsf{M}),v_{Y}\Vdash\psi$ and $At(\psi)\subseteq Y$ iff (def.
of $\pi_{a}$ and $\pi_{a}(w_{X})=w_{Y}$), for all $v_{Y}$ with
$(\pi_{a}(w_{X}),v_{Y})\in R_{Ya}$, $L(\mathsf{M}),v_{Y}\Vdash\psi$
and $At(\psi)\subseteq Y$ iff (def. of $\Vdash$) $L(\mathsf{M}),w_{X}\Vdash K_{a}\psi$
and $At(\psi)\subseteq Y$.
\end{proof}

\begin{proposition}
\label{prop:equivalence-2}For any Kripke lattice model $\mathsf{K}$
with $H$-transform $H(\mathsf{K})$, for all $\varphi\in\mathcal{L}$,
for all $w_{X}\in\Omega_{\mathsf{K}}$,
\[
\mathsf{K},w_{X}\Vdash\varphi\text{ iff }H(\mathsf{\mathsf{K}}),w_{X}\vDash\varphi.
\]
\end{proposition}

\begin{proof}
Let $\mathsf{K}=(\mathcal{K}(\mathtt{K}=(W,R,V)),\trianglelefteqslant,\pi)$
with $w_{X}\in\Omega_{\mathsf{K}}$, $\pi_{a}(w_{X})\in W_{Y}$ with
$Y\subseteq At$, and let $H(\mathsf{K})=(\mathcal{S},\preceq,\mathcal{R},\Pi,V_{H(\mathsf{K})})$.
Let $\varphi\in\mathcal{L}$ and proceed by induction on formula complexity.

\emph{Base:} $i)$ $\varphi:=p\in At$ or $ii)$ $\varphi:=\top$.
$i)$ $\mathsf{\mathsf{K}},w_{X}\Vdash p$ iff (def. of $\Vdash$)
$w_{X}\in V_{X}(p)$ with $p\in X$ iff (def. of $H$-transform) $w_{X}\in V_{H(\mathsf{K})}(p)$
iff (def. of $\vDash$) $H(\mathsf{K}),w_{X}\vDash p$. $ii)$ is
trivial.

\emph{Step.} Assume $\psi,\chi\in\mathcal{L}$ satisfy Prop. \ref{prop:equivalence-2}.

$\varphi:=\neg\psi$. There are two cases: $i)$ $At(\psi)\subseteq X$
or $ii)$ $At(\psi)\not\subseteq X$. $i)$ $\mathsf{\mathsf{K}},w_{X}\Vdash\neg\psi$
iff (def. of $\Vdash$) not $\mathsf{K},w_{X}\Vdash\psi$ iff (def.
of $\llbracket\psi\rrbracket$) $w_{X}\not\in\llbracket\psi\rrbracket$
iff (def. of $\llbracket\psi\rrbracket$ and $At(\psi)\subseteq X$)
$w_{X}\in\neg\llbracket\psi\rrbracket$ iff (def. of $\vDash$) $H(\mathsf{K}),w_{X}\vDash\neg\psi$.
$ii)$ is trivial: $\varphi$ is undefined in $(\mathsf{K},w_{X})$
iff it is so in $(H(\mathsf{M}),w_{X})$.

$\varphi:=\psi\wedge\chi$. The case follows by tracing \emph{iff}s
through the definitions of $\Vdash$, $H$-transform, and $\vDash$.

$\varphi:=K_{a}\psi$. $\mathsf{\mathsf{K}},w_{X}\Vdash K_{a}\psi$
iff (def. of $\Vdash$) $\pi_{a}(w_{X})R_{Ya}v_{Y}$ implies $\mathsf{K},v_{Y}\Vdash\varphi$
iff (def. of $\pi_{a}$, i.e. $\pi_{a}(w_{X})=w_{Y}$ and def. of
$I_{a}$), for all $v_{Y}$ such that $(w_{Y},v_{Y})\in R_{Ya}$,
i.e. for all $v_{Y}\in I_{a}(w_{Y})$, $\mathsf{K},v_{Y}\Vdash\varphi$
iff (def. of $\Pi_{a}$, i.e. $\Pi_{a}(w_{X})=I_{a}(\pi_{a}(w_{X})=I_{a}(w_{Y})$)
$\Pi_{a}(w_{X})\subseteq\llbracket\psi\rrbracket$ iff (def. of $\boldsymbol{K}_{a}$)
$w\in\boldsymbol{K}_{a}(\llbracket\psi\rrbracket)$ iff (def. of $\vDash$)
$H(\mathsf{K}),w_{X}\vDash K_{a}\psi$.
\end{proof}

\section{\label{sec:Logic}The HMS Logic of Kripke Lattice Models with Equivalence
Relations}

As we may transition back-and-forth between HMS models and Kripke
lattice models with equivalence relations in a manner that preserve
satisfaction of formula of $\mathcal{L}$, soundness and completeness
of a $\mathcal{L}$-logic is also transferable between the model classes.
We thereby show such results for Kripke lattice models with equivalence
relations as a corollary to results by HMS \cite{HMS2008}.
\begin{definition}
The logic $\Lambda_{HMS}$ is the smallest set of $\mathcal{L}$ formulas
that contains the axioms in, and is closed under the inference rules
of, Table \ref{tab:logic}.
\end{definition}

\begin{table}
\vspace{-24pt}%
\begin{tabular}{|>{\raggedright}p{0.98\textwidth}|}
\hline 
{\small{}All substitution instances of propositional logic, including
the formula $\top$}{\small\par}

{\small{}$A_{a}\neg\varphi\leftrightarrow A_{a}\varphi$\hfill{}(Symmetry)}{\small\par}

{\small{}$A_{a}(\varphi\wedge\psi)\leftrightarrow A_{a}\varphi\wedge A_{a}\psi$\hfill{}(Awareness
Conjunction)}{\small\par}

{\small{}$A_{a}\varphi\leftrightarrow A_{a}K_{b}\varphi$, for all
$b\in Ag$\hfill{} (Awareness Knowledge Reflection)}{\small\par}

{\small{}$K_{a}\varphi\rightarrow\varphi$\hfill{} (T, Axiom of Truth)}{\small\par}

{\small{}$K_{a}\varphi\rightarrow K_{a}K_{a}\varphi$\hfill{} (4,
Positive Introspection Axiom)}{\small\par}

{\small{}From $\varphi$ and $\varphi\rightarrow\psi$, infer $\psi$
\hfill{}(Modus Ponens)}{\small\par}

{\small{}For $\varphi_{1},\varphi_{2},...,\varphi_{n},\varphi$ that
satisfy $At(\varphi)\subseteq\bigcup_{i=1}^{n}At(\varphi_{i})$,}{\small\par}

{\small{}from $\bigwedge_{i=1}^{n}\varphi_{i}\rightarrow\varphi$,
infer $\bigwedge_{i=1}^{n}K_{a}\varphi_{i}\rightarrow K_{a}\varphi$
\hfill{}(RK-Inference)}\tabularnewline
\hline 
\end{tabular}\vspace{2pt}

\caption{\label{tab:logic}Axioms and inference rules of the HMS logic of unawareness,
$\Lambda_{HMS}$.}
\vspace{-24pt}
\end{table}
As the the $L$-transform of an HMS model has equivalence relations,
one may be surprised by the lack of the standard negative introspection
axiom $5:(\neg K_{a}\varphi\rightarrow K_{a}\neg K_{a}\varphi)$ among
the axioms of $\Lambda_{HMS}$. However, including 5 would make collapse
awareness \cite{ModicaRustichini1994}. In \cite{HMS2008}, HMS remarks
that $\Lambda_{HMS}$ implies the weakened version $K_{a}\neg K_{a}\neg K_{a}\varphi\rightarrow(K_{a}\varphi\vee K_{a}\neg K_{a}\varphi)$,
which by the Modica-Rustichini definition of awareness is $K_{a}\neg K_{a}\neg K_{a}\varphi\rightarrow A_{a}\varphi$.
Defining unawareness by $U_{a}\varphi:=\neg A_{a}\varphi$, this again
equates $U_{a}\varphi\rightarrow\neg K_{a}\neg K_{a}\neg K_{a}\varphi$.
Additionally, HMS notes that if $\varphi$ is a theorem, then $A_{a}\varphi\rightarrow K_{a}\varphi$
is a theorem, that $4$ implies introspection of awareness ($A_{a}\varphi\rightarrow K_{a}A_{a}\varphi$),
while $\Lambda_{HMS}$ entails that \emph{awareness is generated by
primitives propositions}, i.e., that $A_{a}\varphi\leftrightarrow\bigwedge_{p\in At(\varphi)}A_{a}p$
is a theorem. The latter two properties entails that HMS awareness
is \emph{propositionally determined}, in the terminology of \cite{HalpernRego2008}.

Using the above given notion of validity and standard notions of proof,
soundness and completeness, HMS \cite{HMS2008} state that, as standard,
\begin{lemma}
\label{lem:4.12}The logic $\Lambda_{HMS}$ is complete with respect
to a class of structures $\mathfrak{S}$ iff every set of $\Lambda_{HMS}$
consistent formulas is satisfied in some $\mathfrak{s}\in\mathfrak{S}$.
\end{lemma}

\noindent Let $\boldsymbol{M}$ be the class of HMS modes. Using a
canonical model, HMS show:
\begin{theorem}
[\hspace{-1pt}\hspace{1pt}\cite{HMS2008}]\label{thm:HMS-completeness}$\Lambda_{HMS}$
is sound and complete with respect to $\boldsymbol{M}$.
\end{theorem}

Let $\boldsymbol{KLM}_{EQ}$ be the class of Kripke lattice models
where all accessibility relations are equivalence relations. As a
corollary to Theorem \ref{thm:HMS-completeness} and our transformation
and equivalence results, we obtain
\begin{theorem}
\label{thm: HMS-KLM sound complete }$\Lambda_{HMS}$ is sound and
complete with respect to $\boldsymbol{KLM}_{EQ}$.
\end{theorem}

\begin{proof}
Soundness: The axioms of $\Lambda_{HMS}$ are valid in $\boldsymbol{KLM}_{EQ}$.
We show the contrapositive. Let $\varphi\in\mathcal{L}$. If $\varphi$
is not valid in $\boldsymbol{KLM}_{EQ}$, then for some $\mathsf{K}\in\boldsymbol{KLM}_{EQ}$
and some $w$ from $\mathsf{K}$, $\mathsf{K},w\Vdash\neg\varphi$.
Then its $H$-transform $H(\mathsf{K})$ is an HMS model cf. Prop.
\ref{prop:H-transform-is-HMS}, and $H(\mathsf{K}),w\vDash\neg\varphi$
cf. Prop. \ref{prop:equivalence-2}. Hence $\varphi$ is not valid
in the class of HMS models. The same reasoning implies that the $\Lambda_{HMS}$
inference rules preserve validity.

Completeness: Assume $\Phi\subseteq\mathcal{L}$ is a consistent set,
and let $\mathfrak{M}$ be the canonical model of HMS, with $\mathfrak{w}$
a state in $\mathfrak{M}$ that satisfies $\Phi$. This exists, cf.
\cite{HMS2008}. By Prop.s \ref{prop:L-transform-is-KLM} and \ref{prop:L-transform-has-EQ-R},
$L(\mathfrak{M})$ is in $\boldsymbol{KLM}_{EQ}$. By Prop. \ref{prop:equivalence},
for all $v\in\ell(\mathfrak{w})$, $L(\mathfrak{M}),v\Vdash\Phi$.
By Lemma \ref{lem:4.12}, $\Lambda_{HMS}$ is thus complete with
respect to $\boldsymbol{KLM}_{EQ}$.
\end{proof}

\section{\label{sec:FH}The FH Model}

 We next turn to the syntax-based FH model, the first model for awareness
in the field of logic, introduced in \cite{HalpernFagin88}. In \cite{HalpernFagin88}
the models are referred to as \emph{awareness structures}. We propose
transformations between these structures and Kripke lattice models,
to show the relations between the two model classes. The transformations
preserve formula satisfaction.

In the literature, the FH model is said to adopt a \emph{syntactic}
\emph{approach}, as it models awareness by adding a syntactic \emph{awareness
function} $\mathcal{A}_{a}$ to standard Kripke models $(W,R,V)$
for $At'\subseteq At$.\footnote{In \cite{HalpernRego2008}, $R$ is not defined as assigning to each
agent $a\in Ag$ a relation $R_{a}$ between states, as we do above,
but as providing a \emph{possibility correspondence} $R'_{a}:W\rightarrow2^{W}$.
As Halpern and Rêgo write, the approaches are equivalent: $R_{a}$
is definable from a possibility correspondence, and \emph{vice versa,
}by taking $v\in R'_{a}(w)$ iff $(w,v)\in R{}_{a}$.

Similarly for the valuation function, which FH defines as $V':W\times At'\rightarrow\{0,1\}$
and we define as $V:At'\rightarrow\mathcal{P}(W)$. The two definitions
are equivalent, as we can define one in terms of the other by taking
$V'(w,p)=1$ iff $w\in V(p)$.}

The language on which FH originally defined the awareness function\textemdash call
it $\mathcal{L}^{KA}$\textemdash includes both an awareness and
an implicit\emph{ }knowledge operators as primitives, as well as an
explicit knowledge operator definable as the conjunction of the two
\cite{HalpernFagin88}.

As we seek to directly establish the Figure \ref{fig:triangles}'s
promised equivalence between FH models and Kripke lattice models with
respect to the HMS language $\mathcal{L}$ (containing only the explicit
knowledge operator $K_{a}$), in this section we use FH models as
a semantics for \emph{$\mathcal{L}$}. This entails letting $\mathcal{A}_{a}$
assign formulas from $\mathcal{L}$, and not $\mathcal{L}^{KA}$.
Additionally, to establish equivalence, we must focus on the special
case of FH models that in which awareness is \emph{propositionally
determined} (cf. Def. \ref{def:FH model}). We introduce $\mathcal{L}^{KA}$
in Section \ref{sec:L^A}, where we show that the FH and Kripke lattice
models are equivalent with respect to that language as well.
\begin{definition}
\label{def:FH model}An \textbf{FH model} for $At'\subseteq At$ is
a tuple $\mathsf{S}=(W,R,V,\mathcal{A})$ where $(W,R,V)$ is a Kripke
model for $At'$, and $\mathcal{A}$ is an \textbf{awareness function}
$\mathcal{A}:Ag\times W\rightarrow2^{\mathcal{L}}$ that assigns to
each agent $a\in Ag$ and world, $w\in W$ a set of formula denoted
$\mathcal{A}_{a}(w)$.

The function $\mathcal{A}$ satisfies
\begin{description}
\item [{\label{FH1:generated by primitive prop}PP}] (\textbf{Awareness
is Generated by Primitive Propositions})\quad{} if for all $a\in Ag$
and $\varphi\in\mathcal{L}$, $\varphi\in\mathcal{A}_{a}(w)$ iff
for all $p\in At(\varphi)$, $p\in(\mathcal{A}_{a}(w)\cap At')$.
\item [{\label{FH2:Know Awareness}KA}] \textbf{(Agents Know What They
are Aware of})\quad{}if or all $a\in Ag$, $(w,v)\in R_{a}$ implies
$\mathcal{A}_{a}(w)=\mathcal{A}_{a}(v)$.
\end{description}
If $\mathcal{A}_{a}$ satisfies \ref{FH1:generated by primitive prop}
and \ref{FH2:Know Awareness}, then $\mathsf{S}$ is\textbf{ propositionally}~\textbf{determined.}

\noindent $\mathsf{S}$ is called partitional (resp. reflexive, transitive)
iff for each $a\in Ag$, $R_{a}$ is an equivalence relation (resp.
reflexive, transitive).
\end{definition}

If no restrictions are are applied to $\mathcal{A}_{a}$, then an
agent can be aware of an arbitrary set of formulas. For example, for
$w\in W$, we may have both $\neg\varphi\in\mathcal{A}_{a}(w)$ and
$\varphi\in\mathcal{A}_{a}(w)$, or $\varphi\wedge\psi\in\mathcal{A}_{a}$
without having $\psi\wedge\varphi\in\mathcal{A}_{a}(w)$ \cite{HalpernFagin88}.
That awareness is generated by primitive proposition ensures that,
at every state, the agent is aware of all and only the formulas that
are formed from some subset of the set of atoms $At$.

Halpern \cite{halpern2001alternative} shows that if $\mathcal{A}_{a}$
satisfies this property, then in a partitional awareness structures
$\mathsf{S}$, the awareness operator can be characterized as Modica-Rustichini
and HMS suggest \cite{ModicaRustichini1994,ModicaRustichini1999,HMS2006},
i.e. so that any FH model validates $A_{i}\leftrightarrow(K_{i}\vee(\neg K_{i}\wedge K_{i}\neg K_{i}))$,
when employing the following semantics:
\begin{definition}
\label{def: FH semantics for L}Let $\mathsf{S}=(W,R,V,\mathcal{A})$
be an FH model and let $w\in W$. Satisfaction of \textup{$\mathcal{L}$}
formulas is given by\vspace{-10pt}
\noindent \begin{center}
\begin{tabular}{lccllllll}
$\mathsf{S},w\vDash\top$ &  & \multicolumn{2}{c}{for all $w\in W$;} &  & $\mathsf{S},w\vDash\varphi\wedge\psi$ &  & iff\enskip{} & $\mathsf{S},w\vDash\varphi$ and $\mathsf{S},w\vDash\psi$;\tabularnewline
$\mathsf{S},w\vDash p$ &  & iff\enskip{} & $w\in V(p)$; & \hspace{1.5em}\quad{} & $\mathsf{S},w\vDash K_{a}\varphi$ &  & iff\enskip{} & $\varphi\in\mathcal{A}_{a}(w)$ and for all $v\in W$\tabularnewline
$\mathsf{S},w\vDash\neg\varphi$ &  & iff\enskip{} & $\mathsf{S},w\not\vDash\varphi$; &  &  &  &  & s.t. $(w,v)\in R_{a}$, $\mathsf{S},v\vDash\varphi$.\tabularnewline
\end{tabular}
\par\end{center}

\end{definition}

The FH semantics for $\mathcal{L}$ over FH models is defined as standard
in epistemic logic, except for the knowledge operator $K_{a}$, with
$a\in Ag$. In standard epistemic logic, $K_{a}$ represents \emph{implicit}
knowledge, semantically defined as the formulas that are satisfied
in all the worlds the agent has access to. In the FH semantics, $K_{a}$
represents \emph{explicit }knowledge, namely the formulas that $a$
implicitly knows and that belong to $a$'s awareness set.

\section{\label{sec:Moving btwn FH and Kripke lattices}Moving between FH
Models and Kripke Lattices: Transformations and $\mathcal{L}$-Equivalence}

To clarify the relationship between FH models and Kripke lattice models,
we introduce transformations between the two model classes, showing
that a model from one class encodes the structure of a model from
the other. As both structure types are based on Kripke models (FH
models are Kripke models augmented with an awareness function, and
Kripke lattices are spawned from a Kripke model), the main task in
moving from FH models to Kripke lattices is to compose the awareness
map $\pi_{a}$ by extracting semantic information from the syntactically
defined awareness function $\mathcal{A}_{a}$. Conversely, moving
from Kripke lattices to FH models requires to compose $\mathcal{A}_{a}$
by extracting syntactic information from the semantically defined
$\pi_{a}$.

\subsection{From FH Models to Kripke Lattice Models}
\begin{definition}
\label{def:K-transform}Let $\mathsf{S}=(W,R,V,\mathcal{A})$ be an
FH model for $At$. The \textbf{$K$-transform model of $\mathsf{S}$}
is $K(\mathsf{S})=(\mathcal{K}(\mathtt{K}),\trianglelefteqslant,\pi)$
with Kripke model $\mathtt{K}=(W',R',V')$ for $At$ given by\smallskip{}

\noindent \quad{}$W'=W$;

\noindent \quad{}$R'=R$;

\noindent \quad{}$V'(p)=V(p)$, for every $p\in At$;

\noindent \quad{}$\pi$ assigns to each $a\in Ag$ a map $\pi_{a}:\Omega_{K(\mathsf{S})}\rightarrow\Omega_{K(\mathsf{S})}$
s.t., for all $w_{X}\in\Omega_{K(\mathsf{S})}$,

\noindent \quad{}$\pi_{a}(w_{X})=w_{Z}$ with $Z=X\cap Y$ and $Y=\{p\in At:p\in\bigcup_{\varphi\in\mathcal{A}_{a}(w)}At(\varphi)\}$.
\end{definition}

The $K$-transform takes the Kripke model on which the FH model is
based and spawns a lattice from there. Then, it constructs the awareness
map $\pi_{a}$ by extracting, for every world $w$, the set $Y$ of
atoms occurring in formulas in $\mathcal{A}_{a}(w)$, and relating
each world $w_{X}$ in the Kripke lattice to its weakly less expressive
counterpart $w_{Z}$ if, and only if, the vocabulary $Z$ is the subset
of $Y$ that is defined in $X$ (and thus expressible in $w_{X}$).
\begin{remark}
The $K$-transform model $K(\mathsf{S})$ is well-defined as the object
$\mathtt{K}=(W',R',V')$ is clearly a Kripke model for $At$.
\end{remark}

\begin{proposition}
\label{prop:K-transform-is-KLM}For any FH model $\mathsf{S}$ where
agents know what they are aware of, its $K$-transform $K(\mathsf{S})$
is a Kripke lattice model.
\end{proposition}

\begin{proof}
Let $\mathsf{S}=(W,R,V,\mathcal{A})$ be an FH model. We show that
$K(\mathsf{S})=(\mathcal{K}(\mathtt{K}=(W',R',V'),\trianglelefteqslant,\pi)$
is a Kripke lattice model by showing that $\pi_{a}$ satisfies the
three properties of an awareness map:

\emph{\ref{our1-DownwardsProjection}}: Consider some $w_{X}\in\Omega_{K(\mathsf{S})}$.
By def. of $K$-transform, $w\in W'=W$ and, for all $a\in Ag$, $\pi_{a}(w_{X})=w_{Z}$,
with $Z=X\cap Y$ and $Y=\{p\in At:p\in At(\varphi),\varphi\in\mathcal{A}_{a}(w)\}$.
Thus, $Z\subseteq X$, i.e. \ref{our1-DownwardsProjection} holds
for $\pi_{a}$.

\emph{\ref{our2-Intro.Idem}}: Let $\pi_{a}(w_{X})=w_{Z}$, and consider
some $v_{Z}\in\Omega_{K(\mathsf{S})}$ such that $v_{Z}\in I_{a}(w_{Z})$,
with $I_{a}(w_{Z})=\{v_{Z}\in\Omega_{K(\mathsf{S})}:(w_{Z},v_{Z})\in R'_{Za}\}$.
Then, $(w,v)\in R'_{a}$, by def. of Kripke lattice model, and $(w,v)\text{\ensuremath{\in}}R{}_{a}$,
by construction of $K(\mathsf{S})$. By \ref{FH2:Know Awareness},
it follows that $\mathcal{A}_{a}(w)=\mathcal{A}_{a}(v)$, and so $\{p\in At:p\in\bigcup_{\varphi\in\mathcal{A}_{a}(w)}At(\varphi)\}=Y=\{p'\in At:p'\in\bigcup_{\varphi'\in\mathcal{A}_{a}(v)}At(\varphi')\}$.
Then, by construction of $\pi_{a}$ in $K(\mathsf{S})$, $\pi_{a}(v_{Z})=v_{Z'}$
with $Z'=Z\cap Y=(X\cap Y)\cap Y=Z$. Hence, $\pi_{a}(v_{Z})=v_{Z}$,
and since by assumption $v_{Z}\in I_{a}(w_{Z})$, then \ref{our2-Intro.Idem}
holds for $\pi_{a}$.

\emph{\ref{our3-NoSurpises}}: Let $\pi_{a}(w_{X})=w_{Z}$. Then,
$Z=X\cap Y$ with $Y=\{p\in At:p\in\bigcup_{\varphi\in\mathcal{A}_{a}(w)}At(\varphi)\}$,
by construction of $K(\mathsf{S})$. Consider some $X'\subseteq X$.
By def. of $K(\mathsf{S})$, $\pi_{a}(w_{X'})=w_{Z''}$ with $Z''=X'\cap Y$.
We have two cases: $i)$ $X'\subseteq Z$; $ii)$ $Z\subset X'$.
$i)$ As $Z=X\cap Y$, then by $X'\subseteq Z$, $X'\subseteq(X\cap Y)$
and so $X'\subseteq Y$. Then, $Z''=X'$, and since $X'\subseteq Z$,
then $Z''=X'\cap Z$. Hence, $\pi_{a}(w_{X'})=w_{X'\cap Z}$ and \ref{our3-NoSurpises}
holds for $\pi_{a}$. $ii)$ As $Z\subset X'$, then $X'\cap Z=Z$.
So to show that $\pi_{a}(w_{X'})=w_{X'\cap Z}$, we need to show that
$Z=Z''=X'\cap Y$, i.e. $1)$ $Z\subset(X'\cap Y)$ and $2)$ $(X'\cap Y)\subseteq Z$.
$1)$ By assumption $Z\subset X'$, so $(X\cap Y)\subset X'$, and
clearly $(X\cap Y)\subseteq Y$. Thus, $(X\cap Y)\subset(X'\cap Y)$,
i.e. $Z\subset(X'\cap Y)$; $2)$ By assumption, $X'\subseteq X$,
so $(X'\cap Y)\subseteq X$, and clearly $(X'\cap Y)\subseteq Y$.
Thus, $(X'\cap Y)\subseteq(X\cap Y)$, i.e. $(X'\cap Y)\subseteq Z$.
Hence, $Z=X'\cap Y$ and $\pi_{a}(w_{X'})=w_{X'\cap Z}$, so \ref{our3-NoSurpises}
holds for\,$\pi_{a}$.
\end{proof}

\begin{remark}
In Prop.\emph{ }\ref{prop:K-transform-is-KLM}, the requirement that
agents know what they are aware of is necessary to match the Introspective
Idempotency property of $\pi_{a}$.
\end{remark}

\begin{remark}
Prop. \ref{prop:K-transform-is-KLM} does not require that awareness
is generated by primitive propositions, as the $K$-transform extracts
atomic information from by checking subformulas of $\mathcal{A}_{a}$.
Hence $\mathcal{A}_{a}$ need not itself contain atoms.
\end{remark}

\begin{remark}
\label{remark: generalize prop} Prop. \ref{prop:K-transform-is-KLM}
further does not require assuming that the FH model is partitional,
reflexive, or transitive. These properties are however clearly preserved
by $K$-transforms: for any FH model $\mathsf{S}=(W,R,V,\mathcal{A})$,
where, for all $a\in Ag$, $R_{a}$ satisfies $C\subseteq\{$partitional,
reflexive, or transitive$\}$, and agents are aware of their own awareness,
its $K$-transform $K(\mathsf{S})=(\mathcal{K}(\mathtt{K}=(W',R',V'),\trianglelefteqslant,\pi)$
is a Kripke lattice model, where $R'_{a}$ satisfies $C$ as well.
\end{remark}

\subsection{From Kripke Lattice Models to FH Models}

In the following, we define the $FH$-transform, which encodes a Kripke
lattice model as an FH model. The core idea is to take the top model
of the lattice and augment it with an awareness function. The latter
assigns, for each agent, the set of all formulas from $\mathcal{L}$
that mention any of the atoms appearing in the model of the lattice
where the awareness image of the agent resides.
\begin{definition}
\label{def:F-transform}Let $\mathsf{K}=(\mathcal{K}(\mathtt{K}=(W,R,V)),\trianglelefteqslant,\pi)$
be a Kripke lattice model for $At$. The \textbf{$FH$-transform}
of \textbf{$\mathsf{K}$} is $FH(\mathsf{K})=(W',R',V',\mathcal{A})$
where

\noindent $W'=W$;

\noindent $R'=R$;

\noindent $V'(p)=V(p)$ for all $p\in At$;

\noindent $\mathcal{A}$ is such that, for all $a\in Ag$, $\mathcal{A}_{a}\in(2^{\mathcal{L}})^{W}$
with $\mathcal{A}_{a}(w)=\{\varphi\in\mathcal{L}:At(\varphi)\subseteq Y\subseteq At\text{ for the }Y\text{ such that }\pi_{a}(w)=w_{Y}\}$.
\end{definition}

We show that the $FH$-transform produces a model of the FH class. 
\begin{proposition}
\label{prop:F-transform-is-FH}For any Kripke lattice model $\mathsf{K}$,
the $FH$-transform $FH(\mathsf{K})$ is an FH model where awareness
is generated by primitive propositions.
\end{proposition}

\begin{proof}
Let $\mathsf{K}$ be $\mathsf{K}=(\mathcal{K}(\mathtt{K}=(W,R,V)),\trianglelefteqslant,\pi)$,
where $\mathtt{K}$ is a Kripke model for $At$. Let $FH(\mathsf{K})=(W',R',V',\mathcal{A})$
be its $FH$-transform. Clearly, $(W',R',V')$ is a Kripke model for
$At$, as $\mathtt{K}$ is so. Moreover, $\mathcal{A}$ is an awareness
function, as for each $a\in Ag$ and $w\in W'$, $\mathcal{A}_{a}(w)$
is a set of formulas from $\mathcal{L}$. Lastly, we show that $\mathcal{A}_{a}(w)$
is generated by primitive propositions. Let $\pi_{a}(w)=w_{Y}$ for
some $Y\subseteq At$, $a\in Ag$, $w\in W'$ and consider $\varphi\in\mathcal{L}$.

$(\Rightarrow)$ Suppose that $\varphi\in\mathcal{A}_{a}(w)$. Then,
by construction of $\mathcal{A}$, for all $p\in At(\varphi)$, $p\in\mathcal{A}_{a}(w)$,
and since $At(\varphi)\subseteq At$ then $p\in(\mathcal{A}_{a}(w)\cap At)$.

$(\Leftarrow)$ Suppose that for all $p\in At(\varphi)$, $p\in(\mathcal{A}_{a}(w)\cap At)$.
Then, by construction of $\mathcal{A}$, $At(\varphi)\subseteq Y\subseteq At$,
where $Y$ is the unique set such that $\pi_{a}(w)=w_{Y}$, and so
$\varphi\in\mathcal{A}_{a}(w)$. Hence, for all $a\in Ag$, $\varphi\in\mathcal{A}_{a}(w)$
iff for all $p\in At(\varphi)$, $p\in(\mathcal{A}_{a}(w)\cap At)$.
Thus, $FH(\mathsf{K})$ is an FH model where awareness is generated
by primitive propositions.
\end{proof}

\begin{remark}
The requirement that awareness is generated by primitive propositions
is needed as the $FH$-transform constructs $\mathcal{A}_{a}$ by
collecting atomic information from $a$'s awareness image and then
setting $\mathcal{A}_{a}$ to be exactly the $\mathcal{L}$ sublanguage
built from these these atoms. The resulting awareness notion is thus
propositionally generated.
\end{remark}

\begin{remark}
As with $K$-transforms (cf. Remark \ref{remark: generalize prop}),
also $FH$-transforms preserve relation properties: for any Kripke
lattice model $\mathsf{K}=(\mathcal{K}(\mathtt{K}=(W,R,V),\trianglelefteqslant,\pi)$
where, for all $a\in Ag$, $R_{a}$ satisfies $C\subseteq\{$partitional,
reflexive, or transitive$\}$ and awareness is generated by primitive
propositions, the $FH$-transform $FH(\mathsf{K})=(W',R',V',\mathcal{A})$
is an FH model where $R'_{a}$ satisfies $C$ as well.
\end{remark}

\subsection{\label{sec:FH semantics}The $\mathcal{L}$-Equivalence of FH and
Kripke Lattice Models}

$K$- and $FH$-transforms not only produce models of the correct
class, but also preserve finer details, as any model and its transform
satisfy the same $\mathcal{L}$ formulas.
\begin{proposition}
\label{prop:equivalence FH-K(FH) wrt L}For any FH model $\mathsf{S}$
satisfying \ref{FH2:Know Awareness}, with $K$-transform $K(\mathsf{S})$,
for all $\varphi\in\mathcal{L}$, for all $w\in W$ and for all $w_{X}\in\Omega_{K(\mathsf{S})}$
with $X\supseteq At(\varphi)$, 
\[
\mathsf{S},w\vDash\varphi\text{ iff }K(\mathsf{S}),w_{X}\Vdash\varphi.
\]
\end{proposition}

\begin{proof}
Let $\mathsf{S}=(W,R,V,\mathcal{A})$ be an FH model and let $K(\mathsf{S})=(\mathcal{K}(\mathtt{K}),\trianglelefteqslant,\pi)$
be its $K$-transform, with $\mathtt{K}=(W',R',V')$. The proof is
by induction on formula complexity. Let $\varphi\in\mathcal{L}$,
$w\in W$ and $w_{X}\in\Omega_{K(\mathsf{S})}$, with $At(\varphi)\subseteq X$
(clearly at least one such $w_{X}$ exists: $\mathsf{S}$ and $K(\mathsf{S})$
are defined for the same set of atoms, $K(\mathsf{S})$ is spawned
from a Kripke model $\mathtt{K}$ that is identical to $(W,R,V)$,
and there is a model for every $X'\subseteq At$).

\emph{Base:} $i)$ $\varphi:=p\in At$ or $ii)$ $\varphi:=\top$.
$i)$ $\mathsf{S},w\vDash p$ iff (def. of $\vDash$) $w\in V(p)$
iff (def. of $K(\mathsf{S})$) $w\in V'(p)$ and $w\in V'_{X}(p)$
such that $At(\varphi)\subseteq X$ iff (def. of $\Vdash$) $K(\mathsf{S}),w_{X}\Vdash\varphi$
with $At(\varphi)\subseteq X$. $ii)$ is trivial.

\emph{Step.} Assume $\psi,\chi\in\mathcal{L}$ satisfy Prop. \ref{prop:equivalence FH-K(FH) wrt L}.

The cases in which $\varphi:=\neg\psi$ or $\varphi:=\psi\wedge\chi$
follow by tracing \emph{iff}s through the definitions of $\vDash$,
$V$, $K$-transform, $\Vdash$, and by inductive hypothesis.

$\varphi:=K_{a}\psi$. $\mathsf{S},w\vDash K_{a}\psi$ iff (def. of
$\vDash$) $(i)$ $\psi\in\mathcal{A}_{a}(w)$ and $(ii)$ for all
$v\in W$ such that $(w,v)\in R_{a}$, $\mathsf{S},v\vDash\psi$.
By def. of $K$-transform and assumption that $At(\varphi)\subseteq X$,
$(i)$ is the case iff $At(\psi)\subseteq Y=\{p\in At:p\in\bigcup_{\varphi'\in\mathcal{A}_{a}(w)}At(\varphi')\}$
and $\pi_{a}(w_{X})=w_{Z}$ with $Z=X\cap Y$, and by $At(\psi)\subseteq X$,
then $At(\psi)\subseteq Z$. By def. of Kripke lattice model and $At(\psi)\subseteq Z$,
$(ii)$ is the case iff for all $v_{Z}\in W$ such that $(w_{Z},v_{Z})\in R_{Za}$,
$\mathsf{S},v_{Z}\vDash\psi$. Hence, by def. of $\Vdash$ and assumption
that $At(\varphi)\subseteq X$, $K(\mathsf{S}),w_{X}\Vdash K_{a}\psi$.
\end{proof}

\begin{proposition}
\label{prop:equivalence KLM - F(KLM) wrt L}For any Kripke lattice
model $\mathsf{K}$ with $FH$-transform $FH(\mathsf{K})$, for all
$\varphi\in\mathcal{L}$, for all $w_{X}\in\Omega_{K(\mathsf{S})}$
with $X\supseteq At(\varphi)$, 
\[
\mathsf{K},w_{X}\Vdash\varphi\text{ iff }FH(\mathsf{\mathsf{K}}),w\vDash\varphi.
\]
\end{proposition}

\begin{proof}
Let $\mathsf{K}=(\mathcal{K}(K=(W,R,V)),\trianglelefteqslant,\pi)$
with $w_{X}\in\Omega_{\mathsf{K}}$, $\pi_{a}(w_{X})\in W_{Y}$ with
$Y\subseteq At$, and let $FH(\mathsf{K})=(W',R',V',\mathcal{A}_{a})$.
Let $\varphi\in\mathcal{L}$, $w_{X}\in\Omega_{\mathsf{K}}$,with
$At(\varphi)\subseteq X$, $w\in W$, and proceed by induction on
formula complexity.

\emph{Base:} $i)$ $\varphi:=p\in At$ or $ii)$ $\varphi:=\top$.
$i)$ $\mathsf{\mathsf{K}},w_{X}\Vdash p$ iff (def. of $\Vdash$)
$w_{X}\in V_{X}(p)$ with $p\in X$ iff (def. of Kripke lattice) $w\in V(p)$
iff (def. of $FH$-transform) $w\in V'(p)$ iff (def. of $\vDash$)
$FH(\mathsf{K}),w\vDash p$. $ii)$ is trivial.

\emph{Step.} Assume $\psi,\chi\in\mathcal{L}$ satisfy Prop. \ref{prop:equivalence KLM - F(KLM) wrt L}.

The cases in which $\varphi:=\neg\psi$ or $\varphi:=\psi\wedge\chi$
follow by tracing \emph{iff}s through the definitions of $\vDash$,
$V$, $K$-transform, $\Vdash$, and by inductive hypothesis.

$\varphi:=K_{a}\psi$. Suppose $\pi_{a}(w_{X})=w_{Y}$ for $Y\subseteq At$.
$\mathsf{\mathsf{K}},w_{X}\Vdash K_{a}\psi$ iff (def. of $\Vdash$)
for all $v_{Y}\in\Omega_{\mathsf{K}}$, such that $w_{Y}R_{Ya}v_{Y}$,
$\mathsf{K},v_{Y}\Vdash\psi$. By def. of Kripke lattice, $\mathsf{K},v_{Y}\Vdash\psi$
iff $At(\psi)\subseteq Y$. 
\begin{claim}
$\pi_{a}(w_{X})=w_{Y}$ with $At(\psi)\subseteq Y$ iff $\pi_{a}(w)=w_{Z}$
with $At(\psi)\subseteq Z$.

We prove the two directions separately.

$(\Rightarrow)$ Suppose not. Then $\pi_{a}(w_{X})=w_{Y}$ with $At(\psi)\subseteq Y$,
and $\pi_{a}(w)=w_{Z}$ with $At(\psi)\not\subseteq Z$. As $X\subseteq At$,
then by \ref{our3-NoSurpises} $\pi_{a}(w_{X})=w_{X\cap Z}$, and
$At(\psi)\not\subseteq X\cap Z$. But $\pi_{a}(w_{X})=w_{Y}$, so
$Y=X\cap Z$, and since $At(\psi)\subseteq Y$, then $At(\psi)\subseteq Z$,
which contradicts our initial assumption. Hence, $\pi_{a}(w)=w_{Z}$
with $At(\psi)\subseteq Z$.

$(\Leftarrow)$ Suppose that $\pi_{a}(w)=w_{Z}$ with $At(\psi)\subseteq Z$.
As $X\subseteq At$, then by \ref{our3-NoSurpises}, $\pi_{a}(w_{X})=w_{X\cap Z}$.
By assumption $At(\varphi)\subseteq X$, and so $At(\varphi)\subseteq(X\cap Z)$.
Hence, $\pi_{a}(w_{X})=w_{Y}$ with $At(\psi)\subseteq Y$.
\end{claim}

So by Claim, $\pi_{a}(w_{X})=w_{Y}$ iff $\pi_{a}(w)=w_{Z}$ with
$At(\psi)\subseteq Z$ iff (def. of $FH$-transform) $\psi\in\mathcal{A}_{a}(w)$.
By def. of Kripke lattice, for all $v_{Y}\in\Omega_{\mathsf{K}}$,
such that $w_{Y}R_{Ya}v_{Y}$, $\mathsf{K},v_{Y}\Vdash\psi$ iff for
all $v\in W$ such that $(w,v)\in R_{a}$, $\mathsf{\mathsf{K}},v\vDash\psi$
iff (def. of $FH$-transform) for all $v\in W'$ such that $(w,v)\in R'_{a}$,
$FH(\mathsf{\mathsf{K}}),v\vDash\psi$ iff (def. of $\vDash$ and
$\psi\in\mathcal{A}_{a}(w)$) $FH(\mathsf{K}),w\vDash K_{a}\psi$.
\end{proof}

\begin{remark}
Prop. \ref{prop:equivalence FH-K(FH) wrt L} and Prop. \ref{prop:equivalence KLM - F(KLM) wrt L}
provide us with another path to prove soundness and completeness of
the HMS logic $\Lambda_{HMS}$ over the class $\boldsymbol{KLM}_{EQ}$
of Kripke lattice models with equivalence relations. Soundness follows
by the same proof structure used in the soundness proof of Theorem
\ref{thm: HMS-KLM sound complete } (this time using Prop. \ref{prop:F-transform-is-FH}
and Prop. \ref{prop:equivalence KLM - F(KLM) wrt L}). Completeness
follows by using Halpern and Rêgo \cite{HalpernRego2008} completeness
results of a logic which we call $\Lambda_{FH}$ over partional and
propositionally determined FH models. The logic $\Lambda_{FH}$ is
based on $\mathcal{L}$ and an axiom system which Halpern and Rêgo
show to be equivalent to that of $\Lambda_{HMS}$ from Table \ref{tab:logic}
(see \cite{HalpernRego2008} for details). Therefore, as a corollary
of this and our transformation results Prop. \ref{prop:K-transform-is-KLM}
and Prop. \ref{prop:equivalence FH-K(FH) wrt L}, one can show that
$\Lambda_{HMS}$ is complete with respect to $\boldsymbol{KLM}_{EQ}$.
\end{remark}

\begin{remark}
These proofs ``close the triangle'' of Figure \ref{fig:triangles},
as we have shown that partitional Kripke lattice models, HMS models,
and partitional propositionally determined FH models are all equivalent
with respect to language $\mathcal{L}$.
\end{remark}

\section{\label{sec:L^A}$\mathcal{L}^{KA}$-Equivalence of FH and Kripke
Lattice Models}

As we mentioned, the FH model and the awareness function $\mathcal{A}_{a}$
were originally designed for the logic $\Lambda_{LGA}$ based on the
language $\mathcal{L}^{KA}$, which contains both an implicit knowledge
and an awareness operators as primitive, with an explicit knowledge
operator definable \cite{HalpernFagin88}. Multiple variations of
$\Lambda_{LGA}$ exist in the literature, some including quantification
over objects \cite{BoardChung2006}, formulas \cite{halpern2009,halpern2013reasoning,Agotnes2014},
and even unawareness \cite{vanDitmarchFrench2009a}, alternative operators
informed through cognitive science \cite{Pietarinen2002}, and dynamic
extensions \cite{Velazquez-Quesada2010-BENTDO-8,grossi2015,Hill2010,vanDitmarchFrench2009a}.

In this section, we show that Kripke lattice models are equivalent
to FH models also with respect to $\mathcal{L}^{KA}$. To show this,
we present the language and semantics of $\mathcal{L}^{KA}$ over
FH and Kripke lattice models. From this, the $K$- and $FH$-transformations
allow us to show $\mathcal{L}^{KA}$-equivalence.
\begin{definition}
With $a\in Ag$ and $p\in At,$ define the language $\mathcal{L}^{KA}$
by
\[
\varphi::=\top\mid p\mid\neg\varphi\mid\varphi\wedge\varphi\mid K_{a}\varphi\mid A_{a}\varphi
\]

Define $X_{i}\varphi:=A_{a}\varphi\wedge K_{a}\varphi$.

Let $At(\varphi)=\{p\in At\colon p\text{ is a subformula of }\varphi\}$,
for all $\varphi\in\mathcal{L}^{KA}$.
\end{definition}

\subsection{FH Models as Semantics for $\mathcal{L}^{KA}$}

The semantics for $\mathcal{L}^{KA}$ over FH models is defined as
the semantics for $\mathcal{L}$ given in Def. \ref{def: FH semantics for L},
except for the knowledge operator $K_{a}$, which now represents \emph{implicit}
knowledge, and for the awareness operator $A_{a}$, which is now taken
as primitive.
\begin{definition}
Let $\mathsf{S}=(W,R,V,\mathcal{A})$ be an FH model for $At$ and
let $w\in W$. Satisfaction of \textup{$\mathcal{L}^{KA}$} formulas
on $\mathsf{S}$ is given by Def. \ref{def: FH semantics for L} for
all formulas except
\noindent \begin{center}
\begin{tabular}{llll}
$\mathsf{S},w\vDash K_{a}\varphi$ &  & iff\enskip{} & for all $v$ s.t. $(w,v)\in R_{a}$, $\mathsf{S},v\vDash\varphi$;\tabularnewline
$\mathsf{S},w\vDash A_{a}\varphi$ &  & iff\enskip{} & $\varphi\in\mathcal{A}_{a}(w)$.\tabularnewline
\end{tabular}
\par\end{center}

\end{definition}

Semantics for explicit knowledge $X_{a}$ is then given by the conjunction
of the semantics for $K_{a}$ and $A_{a}$, with $a\in Ag$. 

$K_{a}$ behaves as a classical knowledge operator in epistemic logic,
as it captures formulas that are satisfied in the information cell
of agent $a$. This notion is closed under implication, whereas explicit
knowledge is not necessarily so: an agent $a$ knows something explicitly
only if $a$ is aware of it, so $X_{a}p\wedge((X_{a}p\rightarrow X_{a}q)\wedge\neg X_{a}q)$
is satisfiable at $w\in W$ when $q\not\in\mathcal{A}_{a}(w)$ \cite{HalpernFagin88}.
However, the kind of FH models considered below are such that awareness
is propositionally generated, i.e. they satisfy \ref{FH1:generated by primitive prop}.
In this restricted class of models, explicit knowledge is closed under
implication as well.

\subsection{Kripke Lattice Models as Semantics for $\mathcal{L}^{KA}$}

As for FH models, also the semantics for $\mathcal{L}^{KA}$ over
Kripke lattice models are defined as the semantics for $\mathcal{L}$
given in Def. \ref{def:Kripke lattice semantics for L}, except for
$K_{a}$ and $A_{a}$. 
\begin{definition}
Let\textbf{ }$\mathsf{K}=(\mathcal{K}(\mathtt{K}=(W,R,V)),\trianglelefteqslant,\pi)$
be a Kripke lattice model with $w_{X}\in\Omega_{\mathsf{K}}$. Satisfaction
of \textup{$\mathcal{L}^{KA}$} formulas on $\mathsf{K}$ is given
by Def. \ref{def: FH semantics for L} for all formulas except
\end{definition}

\noindent \begin{center}
\begin{tabular}{llll}
$\mathsf{K},w_{X}\Vdash K_{a}\varphi$ &  & iff\enskip{} & for all $v\in W$ s.t. $(w,v)\in R_{a}$, $\mathsf{K},v\vDash\varphi$\tabularnewline
$\mathsf{K},w_{X}\Vdash A_{a}\varphi$ &  & iff\enskip{} & $\pi_{a}(w_{X})=w_{Y}$\qquad{}\qquad{}and $At(\varphi)\subseteq Y$\tabularnewline
\end{tabular}
\par\end{center}

Since the top model in a Kripke lattice model represents the objective
perspective, then implicit knowledge $K_{a}$ is defined as the information
cell of agent $a$ in that model. The awareness operator semantics
gives rise to a propositionally generated awareness notion, as it
states that agent $a$ is aware of all the formulas that mention any
of the atoms belonging to the vocabulary that describes $a$'s awareness
image.

\subsection{\label{subsec: equivalence LKA}The $\mathcal{L}^{KA}$-Equivalence
of FH and Kripke Lattice Models}

To show the equivalence of FH and Kripke lattice models with respect
to $\mathcal{L}^{KA}$, the definition of $K$- and $FH$-transforms
must be adapted to the language $\mathcal{L}^{KA}$, by replacing
$\mathcal{L}$ with $\mathcal{L}^{KA}$ in Definitions \ref{def:K-transform}
and \ref{def:F-transform}. The results showing that the transformed
models are of the proper classes are straightforward given the proofs
of Section \ref{sec:Moving btwn FH and Kripke lattices}, and are
therefore omitted to the effect that we only state the results showing
that $K$- and $FH$-transforms not only produce models of the correct
class, but also preserve finer details, as any model and its transform
satisfy the same $\mathcal{L}^{KA}$ formulas.
\begin{proposition}
\label{prop:equivalence FH-K(FH) wrt L^A}For any FH model $\mathsf{S}$
satisfying \ref{FH2:Know Awareness}, with $K$-transform $K(\mathsf{S})$,
for all $\varphi\in\mathcal{L}^{KA}$, for all $w\in W$ and for all
$w_{X}\in\Omega_{K(\mathsf{S})}$ with $X\supseteq At(\varphi)$, 

\[
\mathsf{S},w\vDash\varphi\text{ iff }K(\mathsf{S}),w_{X}\Vdash\varphi.
\]
\end{proposition}

\begin{proof}
Let $\mathsf{S}=(W,R,V,\mathcal{A})$ be an FH model and let $K(\mathsf{S})=(\mathcal{K}(\mathtt{K}),\trianglelefteqslant,\pi)$
be its $K$-transform, with $\mathtt{K}=(W',R',V')$. The proof is
by induction on formula complexity. Let $\varphi\in\mathcal{L}^{KA}$,
$w\in W$ and $w_{X}\in\Omega_{K(\mathsf{S})}$. 

\emph{Base:} $i)$ $\varphi:=p\in At$ or $ii)$ $\varphi:=\top$.
$i)$ $\mathsf{S},w\vDash p$ iff (def. of $\vDash$) $w\in V(p)$
iff (def. of $K(\mathsf{S})$) $w\in V'(p)$ and $w\in V'_{X}(p)$
such that $At(\varphi)\subseteq X$ iff (def. of $\Vdash$) $K(\mathsf{S}),w_{X}\Vdash\varphi$
with $At(\varphi)\subseteq X$. $ii)$ is trivial.

\emph{Step.} Assume $\psi,\chi\in\mathcal{L}$ satisfy Prop. \ref{prop:equivalence FH-K(FH) wrt L^A}.

The cases in which $\varphi:=\neg\psi$ or $\varphi:=\psi\wedge\chi$
follow by tracing \emph{iff}s through the definitions of $\vDash$,
$V$, $K$-transform, $\Vdash$, and by inductive hypothesis.

$\varphi:=K_{a}\psi$. $\mathsf{S},w\vDash K_{a}\psi$ iff (def. of
$\vDash$) for all $v\in W$ such that $(w,v)\in R_{a}$, $\mathsf{S},v\vDash\psi$
iff (def. of $K(\mathsf{S})$) for all $v\in W'$ such that $(w,v)\in R'_{a}$,
$K(\mathsf{S}),v\Vdash\psi$ iff (def. of Kripke lattice) for all
$v_{X}\in W'_{X}$ such that $(w_{X},v_{X})\in R'_{Xa}$ and $At(\psi)\subseteq X$,
$K(\mathsf{S}),v_{X}\Vdash\psi$ iff (def. of $\Vdash$) $K(\mathsf{S}),w_{X}\Vdash\psi$.

$\varphi:=A_{a}\psi$. $\mathsf{S},w\vDash A_{a}\psi$ iff (def. of
$\vDash$) $\psi\in\mathcal{A}_{a}(w)$ iff (def. of $K(\mathsf{S})$)
$\pi_{a}(w_{X})=w_{Z}$ with $Z=X\cap Y$ and $Y=\{p\in At:p\in\bigcup_{\varphi\in\mathcal{A}_{a}(w)}At(\varphi)\}$
iff (assumption $At(\varphi)\subseteq X$) $\pi_{a}(w_{X})=w_{Z}$
and $At(\psi)\subseteq Z$ iff (def. of $\Vdash$) $K(\mathsf{S}),w_{X}\Vdash A_{a}\psi$.
\end{proof}

\begin{proposition}
\label{prop:equivalence KLM - F(KLM) wrt L^A}For any Kripke lattice
model $\mathsf{K}$ with $FH$-transform $FH(\mathsf{K})$, for all
$\varphi\in\mathcal{L}^{KA}$, for all $w_{X}\in\Omega_{\mathsf{K}}$
with $X\supseteq At(\varphi)$,

\[
\mathsf{K},w_{X}\Vdash\varphi\text{ iff }F(\mathsf{\mathsf{K}}),w\vDash\varphi.
\]
\end{proposition}

\begin{proof}
Let $\mathsf{K}=(\mathcal{K}(K=(W,R,V)),\trianglelefteqslant,\pi)$
with $w_{X}\in\Omega_{\mathsf{K}}$, $\pi_{a}(w_{X})\in W_{Y}$ with
$Y\subseteq At$, and let $FH(\mathsf{K})=(W',R',V',\mathcal{A}_{a})$.
Let $\varphi\in\mathcal{L}^{KA}$, $w_{X}\in\Omega_{\mathsf{K}}$,with
$At(\varphi)\subseteq X$, $w\in W$, and proceed by induction on
formula complexity.

\emph{Base:} $i)$ $\varphi:=p\in At$ or $ii)$ $\varphi:=\top$.
$i)$ $\mathsf{\mathsf{K}},w_{X}\Vdash p$ iff (def. of $\Vdash$)
$w_{X}\in V_{X}(p)$ with $p\in X$ iff (def. of Kripke lattice) $w\in V(p)$
iff (def. of $FH$-transform) $w\in V'(p)$ iff (def. of $\vDash$)
$FH(\mathsf{K}),w\vDash p$. $ii)$ is trivial.

\emph{Step.} Assume $\psi,\chi\in\mathcal{L}$ satisfy Prop. \ref{prop:equivalence KLM - F(KLM) wrt L^A}.

The cases in which $\varphi:=\neg\psi$ or $\varphi:=\psi\wedge\chi$
follow by tracing \emph{iff}s through the definitions of $\vDash$,
$V$, $K$-transform, $\Vdash$, and by inductive hypothesis.

$\varphi:=K_{a}\psi$. $\mathsf{\mathsf{K}},w_{X}\Vdash K_{a}\psi$
iff (def. of $\Vdash$) for all $v\in W$ such that $(w,v)\in R_{a}$,
$\mathsf{K},v\vDash\varphi$ iff (def. of $FH$-transform) for all
$v\in W'$ such that $(w,v)\in R'_{a}$, $FH(\mathsf{\mathsf{K}}),v\vDash\psi$
iff (def. of $\vDash$) $FH(\mathsf{K}),w\vDash K_{a}\psi$.

$\varphi:=A_{a}\psi$. $\mathsf{\mathsf{K}},w_{X}\Vdash A_{a}\psi$
iff (def. of $\Vdash$) $\pi_{a}(w_{X})=w_{Y}$ and $At(\psi)\subseteq Y$
iff (Claim in Prop. \ref{prop:equivalence KLM - F(KLM) wrt L}) iff
$\pi_{a}(w)=w_{Z}$ with $At(\psi)\subseteq Z$ iff (def. of $FH$-transform)
$At(\psi)\subseteq\mathcal{A}_{a}(w)$ and $\psi\in\mathcal{A}_{a}(w)$
iff (def. of $\vDash$) $FH(\mathsf{K}),w\vDash A_{a}\psi$.

\end{proof}

\section{\label{sec:Logic-1}The Logic of General Awareness of Kripke Lattice
Models}

The Logic of General Awareness ($\Lambda_{LGA}$) is built on the
language $\mathcal{L}^{KA}$ and an axiom system for implicit knowledge,
awareness and explicit knowledge which is presented in Table \ref{tab:logic-1}.
Using the \emph{$\mathcal{L}^{KA}$-}equivalence results from Section
\ref{subsec: equivalence LKA}, and the transformations results provided
by Prop. \ref{prop:K-transform-is-KLM} and Prop. \ref{prop:F-transform-is-FH},
we show that the class of Kripke lattice models $\boldsymbol{KLM}$
is sound and complete with respect to $\Lambda_{LGA}$.
\begin{definition}
The logic $\Lambda_{LGA}$ is the smallest set of $\mathcal{L}^{KA}$
formulas that contains the axioms in, and is closed under the inference
rules of, Table \ref{tab:logic-1}.
\end{definition}

\begin{table}
\vspace{-24pt}%
\begin{tabular}{|>{\raggedright}p{0.98\textwidth}|}
\hline 
{\small{}All substitution instances of propositional logic, including
the formula $\top$}{\small\par}

{\small{}$(K_{a}\varphi\wedge(K_{a}\varphi\rightarrow K_{a}\psi))\rightarrow K_{a}\psi$\hfill{}(K,
Distribution)}{\small\par}

{\small{}$X_{a}\varphi\leftrightarrow(K_{a}\varphi\wedge A_{a}\varphi)$\hfill{}
(Explicit Knowledge)}{\small\par}

{\small{}$A_{a}(\varphi\wedge\psi)\leftrightarrow(A_{a}\varphi\wedge A_{a}\psi)$\hfill{}
(A1, Awareness Distribution)}{\small\par}

{\small{}$A_{a}\neg\varphi\leftrightarrow A_{a}\varphi$\hfill{}
(A2, Symmetry)}{\small\par}

{\small{}$A_{a}X_{b}\varphi\leftrightarrow A_{a}\varphi$\hfill{}
(A3, Awareness of Explicit Knowledge)}{\small\par}

{\small{}$A_{a}A_{b}\varphi\leftrightarrow A_{a}\varphi$\hfill{}
(A4, Awareness Reflection)}{\small\par}

{\small{}$A_{a}K_{b}\varphi\leftrightarrow A_{a}\varphi$\hfill{}
(A5, Awareness of Implicit Knowledge)}{\small\par}

{\small{}$A_{a}\varphi\rightarrow K_{a}A_{a}\varphi$\hfill{} (A11,
Awareness Introspection)}{\small\par}

{\small{}$\neg A_{a}\varphi\rightarrow K_{a}\neg A_{a}\varphi$\hfill{}
(A12, Unawareness Introspection)}{\small\par}

{\small{}From $\varphi$ and $\varphi\rightarrow\psi$, infer $\psi$
\hfill{}(Modus Ponens)}{\small\par}

{\small{}From $\varphi$ infer $K_{a}\varphi$ \hfill{}(K-Inference)}\tabularnewline
\hline 
\end{tabular}\vspace{2pt}

\caption{\label{tab:logic-1}Axioms and inference rules of $\Lambda_{LGA}$,
for a propositionally determined notion of awareness.}
\vspace{-24pt}
\end{table}

The axiom system of Table \ref{tab:logic-1} is sound and complete
with respect to propositionally determined FH models, i.e. FH models
that satisfy \ref{FH1:generated by primitive prop} and \ref{FH2:Know Awareness}.
In particular, A1-A5 capture an awareness notion that is generated
by primitive propositions, while A11-A12 are required if agents are
to know what they are aware of \cite{halpern2001alternative,fagin1995reasoning}
(the numbering of the awareness axioms is taken from \cite{halpern2001alternative}).
These two properties are needed to establish the transformations results
of Prop. \ref{prop:K-transform-is-KLM} and Prop. \ref{prop:F-transform-is-FH},
and therefore in the later soundness and completeness proofs.

Let $\boldsymbol{S}$ be the class of propositionally determined FH
models. FH \cite{HalpernFagin88,fagin1995reasoning} argue that:\footnote{We say that these works \emph{argue} for soundness and completeness
of $\Lambda_{LGA}$ with respect to FH models, where $\Lambda_{LGA}$
is based on $\mathcal{L}^{KA}$ which is a language for \emph{knowledge}
(not belief) and awareness, as they do not explicitly provide the
proof. They only state that it is straightforward to provide. The
relevant results argued for in the literature are:

1. In \cite{HalpernFagin88}: soundness and completeness for KD45+Explicit
Knowledge with respect to FH models. It does not specify any such
proof about FH models with the \ref{FH1:generated by primitive prop}
and \ref{FH2:Know Awareness} properties.

2. In \cite{halpern2001alternative}: soundness and completeness for\textbf{
}the single agent version of $\Lambda_{LGA}$ with respect to FH models
(also models satisfying \ref{FH1:generated by primitive prop} and
\ref{FH2:Know Awareness}) is claimed a straightforward generalization
of the soundness and completeness proof for the logic formed on language
$\mathcal{L}$ and the $K$ axiom.

3. In \cite{HalpernRego2008}: says that soundness and completeness
of $\Lambda_{LGA}$ with respect to FH models with \ref{FH1:generated by primitive prop}
and \ref{FH2:Know Awareness} is given by FH. Supposedly, they refer
to \cite{HalpernFagin88}, where such construction is not provided\textemdash see
point 1 in this list.}
\begin{theorem}
[\hspace{-1pt}\hspace{1pt}\cite{HalpernFagin88,fagin1995reasoning}]\label{thm:FH-completeness-1}$\Lambda_{LGA}$
is sound and complete with respect to $\boldsymbol{S}$.
\end{theorem}

Let $\boldsymbol{KLM}$ be the class of all Kripke lattice models,
i.e., without special properties assumed of the accessibility relations.
As a corollary to Theorem \ref{thm:FH-completeness-1}, our transformation,
and $\mathcal{L}^{KA}$-equivalence results, we obtain
\begin{theorem}
$\Lambda_{LGA}$  is sound and complete with respect to $\boldsymbol{KLM}$.
\end{theorem}

\begin{proof}
For both soundness and completeness, the reasoning is analogous to
that provided in \ref{thm: HMS-KLM sound complete }. Soundness uses
Prop. \ref{prop:F-transform-is-FH} adapted to language $\mathcal{L}^{KA}$
and Prop. \ref{prop:equivalence FH-K(FH) wrt L^A}. Completeness uses
Prop. \ref{prop:K-transform-is-KLM} adapted to language $\mathcal{L}^{KA}$,
Prop. \ref{prop:equivalence FH-K(FH) wrt L^A}, and the existence
result of the canonical model construction assumed as existing by
\cite{fagin1995reasoning}. 
\end{proof}

\begin{remark}
The same result can clearly be obtained for the logic generated by
the axioms in Table \ref{tab:logic-1} and the axiom system S5.
\end{remark}

\section{\label{sec:Concluding-Remarks}Concluding Remarks}

This paper has introduced Kripke lattice models as a model class for
epistemic logics with awareness. This model is a Kripke model-based
rendition of the syntax-free HMS model of awareness, and we have shown
that the two model classes are equally general with respect to $\mathcal{L}$,
by defining transformations between the two that preserve formula
satisfaction. A corollary to this result is completeness of the HMS
logic for the introduced model class. Moreover, we have shown that
Kripke lattice models and the syntax-based FH models of awareness
are equally general with respect to $\mathcal{L}$, as well as with
respect to the language $\mathcal{L}^{KA}$. As a corollary, we obtain
that the Logic of General Awareness is complete with respect to the
introduced model class.

There are several issues we would like to study in future work:

In recasting HMS models as a Kripke lattice models, we teased apart
the epistemic and awareness dimensions merged in the HMS possibility
correspondences, and Propositions \ref{prop:L-transform-is-KLM},
\ref{prop:L-transform-has-EQ-R} and \ref{prop:H-transform-is-HMS}
about $L$- and $H$-transforms show that the HMS properties are satisfied
iff each $\pi_{a}$ satisfies \ref{our1-DownwardsProjection}, \ref{our2-Intro.Idem}
and \ref{our3-NoSurpises}, and each $R_{a}$ is an equivalence relation.
For a more fine-grained property correspondence, the propositions'
proofs show that each property of one model is entailed by a strict
subset of the properties of the other. In some cases, the picture
emerging is fairly clear: e.g., HMS' \ref{HMS1:confinement} is shown
only using the restrictions lattice construction (RLC) plus \ref{our1-DownwardsProjection}
and \emph{vice versa}; \ref{HMS5:PPK} uses only \ref{our3-NoSurpises}
and RLC, while \ref{HMS5:PPK} and \ref{HMS1:confinement} entail
\ref{our3-NoSurpises}. In other cases, the picture is more murky,
e.g., when we use \ref{HMS3:stationarity}, \ref{HMS4:PPI} and \ref{HMS5:PPK}
to show the seemingly simple symmetry of $R_{a}$. We think it would
be interesting to decompose properties on both sides to see if clearer
relationships arise.

There are two issues with redundant states in Kripke lattice models.
One concerns redundant restrictions, cf. Remark \ref{rem:KLM-redundant-states},
which may be solved by working with a more general model class, where
models may also be based on sub-orders of the restriction lattice.
A second one concerns redundant states. For example, in Figure~\ref{fig:Kripke-lattice},
$\mathtt{K}_{\emptyset}$ contains three `identical' states where
no atoms have defined truth values\textemdash $\mathtt{K}_{\emptyset}$
is bisimilar to a one-state Kripke model. As bisimulation contracting
each $\mathtt{K}_{X}$ may collapse states from which awareness maps
differ, one must define a notion of bisimulation that takes awareness
maps into consideration (notions of bisimulation for other awareness
models exists, e.g. \cite{vanDitmarsch2013}). Together with a more
general model class definition, this could hopefully solve the redundancy
issues.

Kripke lattice models are $\mathcal{L}^{KA}$-equivalent to FH models,
but it is an open issue how HMS models relate to both Kripke lattices
and FH models with respect to that language, cf. the question marks
in Figure \ref{fig:triangles} in the introduction. As $\mathcal{L}^{KA}$
contains an implicit knowledge operator, but HMS models contain no
objective perspective, studying that relation would seemingly mainly
entail exploring how to capture the objective perspective in HMS models.
It is an open question if and how HMS may serve as a semantics for
$\mathcal{L}^{KA}$ in a manner that will entail $\mathcal{L}^{KA}$-equivalence
with FH models and Kripke lattices.

The HMS logic is complete for HMS models and for Kripke lattice models
with equivalence relations. \cite{HalpernRego2008} prove completeness
for HMS models using a standard validity notion, a `$\varphi$ is
at least as expressive as $\psi$' operator and variants of axioms
$T$, $4$ and $5$. We are very interested in considering this system
and its weaker variants for Kripke lattice models, also with less~assumptions~on~the~relations.

Finally, issues of dynamics spring forth: first, whether existing
awareness dynamics may be understood on Kripke lattice models; second,
whether DEL action models may be applied lattice-wide with reasonable
results, and how they compare with other action models for awareness
in the literature \cite{vanDitmarsch2013,vanDitmarchFrench2009a,Velazquez-Quesada2010-BENTDO-8,Hill2010};
and third, whether the $\pi_{a}$ maps may be thought in dynamic terms,
as they map between models.

\subsubsection*{Acknowledgments.}

We thank the organizers of the 3rd DaLí Workshop for the opportunity
to present our work there, and the participants and reviewers of the
conference for their useful and productive comments. The Center for
Information and Bubble Studies is funded by the Carlsberg Foundation.
RKR was partially supported by the DFG-ANR joint project \emph{Collective
Attitude Formation} {[}RO 4548/8-1{]}.

\end{document}